%% file: context_poisoning_aaai27-authors.tex
\documentclass[letterpaper]{article}
\usepackage{aaai2027}
\usepackage{times}
\usepackage{helvet}
\usepackage{courier}
\usepackage[hyphens]{url}
\usepackage{graphicx}
\usepackage{natbib}
\usepackage{caption}
\usepackage{nameref}

\usepackage[utf8]{inputenc}
\usepackage[T1]{fontenc}
\usepackage{booktabs}
\usepackage{amsmath,amssymb,amsfonts,amsthm,mathtools}
\usepackage{nicefrac}
\usepackage{xcolor}
\usepackage{tikz}
\usetikzlibrary{
  arrows.meta,
  positioning,
  calc,
  fit,
  shapes.geometric,
  patterns,
  decorations.pathreplacing
}

\title{Context Poisoning as Extreme-Value Attention Interference in Long-Context Language Models}

\author{%
  Meysam Ghaffari, Nina Fatehi, Bhaskar Sen, Nasim Sabetpour, Carlos Morato
}
\affiliations{%
  Optum, UHG, Minneapolis, MN
}

\newtheorem{definition}{Definition}
\newtheorem{theorem}{Theorem}
\newtheorem{corollary}{Corollary}
\newtheorem{proposition}{Proposition}

\newcommand{\E}{\mathbb{E}}
\newcommand{\Prb}{\mathbb{P}}
\newcommand{\1}{\mathbf{1}}

\newcommand{\argmax}{\operatorname*{arg\,max}}

\newcommand{\D}{\mathcal{D}}

\newcommand{\eps}{\varepsilon}

\begin{document}
\maketitle

\begin{abstract}
Large language models can process increasingly long prompts, yet their
ability to locate and use decisive evidence may degrade as irrelevant or
confusable context is added. We formulate this phenomenon, which we call
\emph{context poisoning}, as extreme-value interference in attention:
the decisive-evidence score is upper-bounded, while the maximum score
among effective distractors grows with their number. Under a softmax
retrieval abstraction, we derive a finite-sample upper bound showing
that maintaining a fixed accuracy target above base rate requires the
evidence margin to scale as $\Omega(\sqrt{\log N})$, where $N$ denotes
the effective distractor count rather than necessarily the raw context
length. The analysis connects long-context degradation to score
aliasing, positional aliasing, and softmax dilution. Controlled
experiments show that retrieval accuracy decreases as total context
grows in the presence of embedded hard negatives, that the same-format
condition produces the largest observed accuracy drop among the tested
distractor constructions at fixed context length, and that retrieval
gating can improve evidence use while its net benefit depends on
preserving evidence recall. These results motivate evidence bottlenecks,
alias-resistant representations, retrieve-then-reason architectures,
verifier-mediated memory, and contrastive anti-poison training.
\end{abstract}

\section{Introduction}

Increasing a language model's context window is not the same as increasing its usable memory. Transformer attention gives every token a path to every other token \citep{vaswani2017attention}, and large autoregressive models can exploit prompts for in-context adaptation \citep{brown2020language}. However, long input does not guarantee robust evidence use. Empirical studies report strong position effects and degradation when relevant information is buried in the middle of a long context \citep{liu2024lost}, sensitivity to irrelevant information \citep{shi2023large}, and benchmark-level gaps between nominal and effective context length \citep{bai2024longbench,an2024leval,hsieh2024ruler,du2025context}. This paper uses the term \emph{context poisoning} for a non-training-data phenomenon: the input itself becomes contaminated by distractors, aliases, and positional artifacts as the window grows.

The practical symptom is familiar. A prompt contains a small set of key facts, then many thousands of irrelevant or weakly related tokens. A model that would answer correctly from the short prompt may answer incorrectly from the long prompt, quote a distractor, or ignore the key fact. Existing long-context systems attack the computational side through recurrence, compression, sparse attention, or fast exact attention \citep{dai2019transformerxl,rae2020compressive,beltagy2020longformer,zaheer2020bigbird,choromanski2021rethinking,katharopoulos2020transformers,dao2022flashattention}. Those methods make long contexts feasible, but feasibility alone does not imply robust extraction. Positional methods such as ALiBi, RoPE, interpolation, and YaRN improve extrapolation behavior \citep{press2022train,su2024roformer,chen2023positionalinterpolation,peng2023yarn}, but long-context failures can persist even within claimed windows.

\paragraph{Contributions.} First, we give formal definitions for context poisoning, score aliasing, and effective distractor count. Second, we prove a theorem: in a simple but diagnostic attention model, the probability that a distractor captures enough attention increases with $N$, and maintaining accuracy requires a signal margin that grows like $\sqrt{\log N}$. Third, we derive design principles that turn the bound into system requirements: reduce $N$, increase the evidence margin, verify evidence before decoding, and regularize aliasing. Fourth, we report two controlled empirical studies in the main paper:
a hard-negative retrieval decay curve on a production model
(\nameref{sec:decay}) and a fixed-length distractor-confusability
ablation across three models (\nameref{sec:ablation}). The appendix
provides a retrieve-then-reason gate experiment at context lengths up
to $512$K tokens (\nameref{sec:gate-experiment}), together with the
full proofs, additional implementation details, complete experimental
protocols and supplemental tables, attention-distribution diagnostics,
a confusable-decoy dose-response analysis, and a lost-in-the-middle
reanalysis. Across these components, the paper connects a finite-sample
failure bound to observable changes in distractor count, distractor
composition, and retrieval recall. We do not claim that every LLM
failure is explained by a single attention head or Gaussian scores. The
theorem is a stress test: if a long-context architecture cannot control
the extreme distractor term in this abstraction, it should not be
expected to remain robust merely by scaling the raw context window.

The remainder of the paper is organized as follows. We first formalize
context poisoning and the effective distractor count, establish the
extreme-value attention bound, examine the mechanisms predicted by the
theory, and derive architectural alternatives motivated by the bound.
\nameref{sec:evaluation-protocol} then introduces the evaluation
framework, followed by two empirical studies in the main paper:
\nameref{sec:decay} and \nameref{sec:ablation}. The appendix contains
the retrieve-then-reason gate experiment
(\nameref{sec:gate-experiment}), the full theoretical proofs and
additional implementation details, the complete
distractor-confusability protocol and supplemental results
(\nameref{app:ablation}), the attention-distribution diagnostics and
confusable-decoy dose-response analysis (\nameref{app:attention}), and
the lost-in-the-middle reanalysis (\nameref{app:lostmiddle}).

\section{Related work}

\paragraph{Efficient long-context architectures.}
Long-context language modeling builds on transformers
\citep{vaswani2017attention}, large-scale pretraining
\citep{devlin2019bert,raffel2020exploring}, and in-context learning
\citep{brown2020language}. Transformer-XL and Compressive Transformers
introduced recurrence and memory compression
\citep{dai2019transformerxl,rae2020compressive}. Longformer, Big Bird,
Performer, and linear attention reduce long-sequence cost
\citep{beltagy2020longformer,zaheer2020bigbird,choromanski2021rethinking,katharopoulos2020transformers},
while FlashAttention improves exact-attention efficiency
\citep{dao2022flashattention}. ALiBi, RoPE, positional interpolation,
and YaRN address positional extrapolation and context-window extension
\citep{press2022train,su2024roformer,chen2023positionalinterpolation,peng2023yarn}.
These methods make long contexts feasible, but feasibility does not
ensure that decisive evidence remains distinguishable from many
competitive candidates.

\paragraph{Selection, compression, and memory.}
Retrieval-augmented generation selects a smaller evidence set before
decoding \citep{lewis2020rag,guu2020realm,izacard2021leveraging}, and
prompt compression removes low-utility tokens
\citep{jiang2023longllmlingua}. Attention-sink methods instead preserve
selected states to stabilize streaming generation
\citep{xiao2024streamingllm}. These approaches expose a shared
tradeoff: reducing the candidate set can lower interference, but an
imperfect selector can discard required evidence. Proposition~\ref{prop:gate}
expresses this tradeoff as replacing $N$ by $K$ effective distractors
while paying a recall penalty $\eta$.

\paragraph{Long-context evaluation and failure analysis.}
Lost-in-the-middle effects show that central evidence can be underused
\citep{liu2024lost}, and irrelevant information can distract reasoning
even when the evidence remains present \citep{shi2023large}. Positional
interventions can partly recover evidence use \citep{zhang2024found}.
LongBench, L-Eval, ZeroSCROLLS, and RULER broaden evaluation across
tasks and context lengths
\citep{bai2024longbench,an2024leval,shaham2023zeroscrolls,hsieh2024ruler}.
Recent controlled work further shows degradation despite controlled or
perfect retrieval \citep{du2025context}, reinforcing the distinction
between nominal capacity, evidence retrieval, and evidence use.

\paragraph{Relation to our contribution.}
Our contribution is to formalize one failure mode as extreme-value
competition among distractor scores. We distinguish raw length from
effective distractor count, derive the margin growth required for
reliable evidence use, and connect mitigations to terms in the bound.
The experiments test complementary implications by varying hard-negative
count, distractor composition at fixed length, and the recall--interference
tradeoff of a retrieval gate.

\section{Formalizing context poisoning}

Let a query $q$ request an answer $y^\star$ from a context $C_N=(x_0,\ldots,x_N)$. The index set of relevant evidence is $R\subseteq\{0,\ldots,N\}$, and $D=\{0,\ldots,N\}\setminus R$ are distractors. A model $f_\theta$ induces an answer distribution $p_\theta(y\mid q,C_N)$ and a score map $s_\theta(q,x_i,i)$ used, explicitly or implicitly, to select evidence. The score includes token semantics and position: for a transformer head one may write
\begin{equation}
 s_i = \frac{\langle W_Q h_q, W_K(h_i+p_i)\rangle}{\sqrt{d}},
\end{equation}
where $h_i$ is a contextual token representation and $p_i$ denotes a positional code or bias.

\begin{definition}[Context poisoning curve]
\label{def:curve}
Fix a task distribution $\D$, a short reference length $N_0$, and an augmentation operator $T_N$ that appends $N-N_0$ irrelevant or nuisance tokens while preserving the answer. The accuracy at length $N$ is
\begin{equation}
 A(N)=\Prb_{(q,C_{N_0},y^\star)\sim\D}\left[\argmax_y p_\theta(y\mid q,T_N(C_{N_0}))=y^\star\right].
\end{equation}
The poisoning curve is $P(N)=A(N_0)-A(N)$. Context poisoning occurs on $[N_0,N_1]$ when $P(N)$ is positive and non-negligible for some $N\leq N_1$, even though $T_N$ does not change the ground-truth answer.
\end{definition}

\begin{definition}[Score aliasing and effective distractor count]
\label{def:aliasing}
For a query $q$, relevant item $r\in R$, and distractor $i\in D$, the pair $(r,i)$ is $\eps$-score-aliased if $|s_i-s_r|\leq \eps$. It is $\eps$-position-aliased if the positional contribution obeys
\begin{equation}
 \left|\langle W_Q h_q, W_Kp_i\rangle-\langle W_Qh_q,W_Kp_r\rangle\right|\leq \eps.
\end{equation}
The effective distractor count $N_{\rm eff}(q,C)$ is the number of distractors whose scores have non-negligible upper-tail probability near a relevant score. In the theorem below, $N_{\rm eff}=N$ by construction; in a system with retrieval or gating, $N_{\rm eff}$ is the number of tokens or chunks surviving the gate.
\end{definition}

Figure~\ref{fig:aliasing} summarizes the issue. Long contexts produce many opportunities for an irrelevant token to become close to, or exceed, the relevant score. This can happen semantically, positionally, or through attention-normalization effects. The key mathematical feature is not the average distractor score; it is the maximum distractor score.

\begin{figure*}[t]
\centering
\resizebox{0.8\linewidth}{!}{%
\begin{tikzpicture}[x=1cm,y=1cm,>=Latex,thick, every node/.style={font=\small}]
  \node[anchor=west,font=\bfseries] at (0,3.35) {A. Score aliasing in a long input};
  \draw[->] (0,2.55) -- (11.2,2.55) node[right] {position};
  \foreach \x/\lab in {0.8/$d_1$,1.9/$d_2$,3.0/$d_3$,4.4/$x_\star$,5.8/$d_4$,7.2/$d_5$,8.6/$d_j$,10.0/$d_N$} {
    \node[draw,rounded corners=1pt,minimum width=.72cm,minimum height=.42cm,fill=gray!10] at (\x,2.55) {\lab};
  }
  \node[draw,rounded corners=1pt,minimum width=.72cm,minimum height=.42cm,fill=white,line width=1.0pt] (sig) at (4.4,2.55) {$x_\star$};
  \node[draw,rounded corners=1pt,minimum width=.72cm,minimum height=.42cm,pattern=north east lines] (alias) at (8.6,2.55) {$d_j$};
  \node[draw,rounded corners=2pt,minimum width=.8cm,minimum height=.38cm] (q) at (6.45,3.55) {$q$};
  \draw[->] (q) -- node[midway,above left=-2pt] {$s_\star$} (sig);
  \draw[->,dashed] (q) -- node[midway,above right=-2pt] {$s_j\approx s_\star$} (alias);
  \node[anchor=west] at (0,2.02) {white box = decisive evidence; hatched box = aliased distractor};
  \node[anchor=west,font=\bfseries] at (0,1.35) {B. Extreme-value interference};
  \draw[->] (0,0.25) -- (11.2,0.25) node[right] {attention logit};
  \draw[smooth,domain=1.0:5.2,samples=70] plot(\x,{0.25+0.55*exp(-0.85*(\x-3.0)*(\x-3.0))});
  \draw[smooth,domain=5.7:10.3,samples=70,line width=0.95pt] plot(\x,{0.25+0.47*exp(-0.55*(\x-8.0)*(\x-8.0))});
  \draw (5.55,0.18) -- (5.55,0.34) node[below=5pt] {$s_\star$};
  \draw (8.0,0.18) -- (8.0,0.34) node[below=5pt] {$M_N$};
  \draw[<->] (5.55,0.92) -- node[above] {required margin} (8.0,0.92);
  \node[anchor=west] at (1.85,0.86) {typical distractors};
  \node[anchor=west] at (8.3,1.05) {$M_N=\max_i Z_i$};
  \node[draw,rounded corners=2pt,align=left,anchor=west] at (0,-0.72) {Poisoning event: $M_N$ is large enough that\\ $\alpha_\star=\exp(s_\star/\tau)/\sum_i\exp(s_i/\tau)$ falls below an evidence threshold.};
\end{tikzpicture}%
}
\caption{Score and positional aliasing in a long context. A relevant token $x_\star$ competes not with an average distractor but with the maximum of many distractors. As $N_{\rm eff}$ grows, the right tail creates aliased distractors whose logits can dominate the softmax denominator.}
\label{fig:aliasing}
\end{figure*}
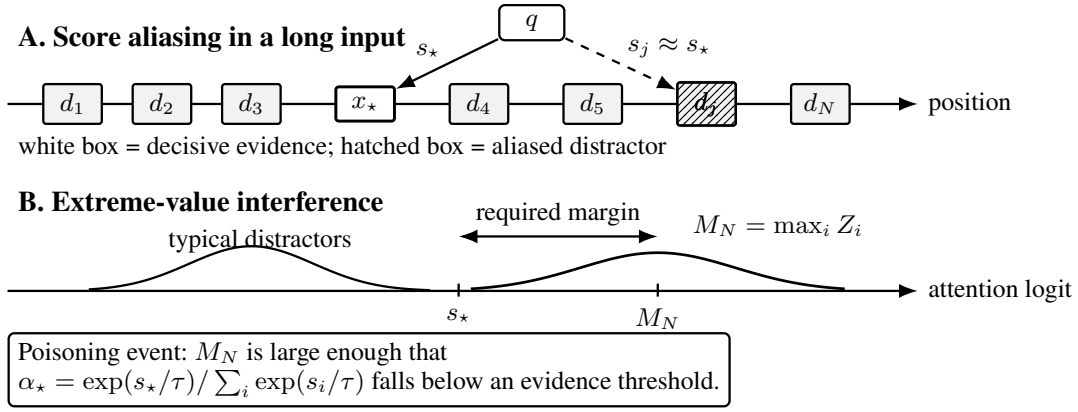

\section{A theorem for attention-mediated poisoning}

The theorem studies a single retrieval bottleneck. This abstraction is intentionally small: many LLM computations are multi-layer and multi-head, yet final answer quality often depends on whether decisive evidence receives enough usable representation mass. The result states that a fixed signal margin cannot protect against an unbounded number of effective distractors.

\begin{definition}[$\rho$-faithful evidence use]
Consider a model with attention mass $\alpha_{\star,N}$ assigned to the decisive evidence. For a threshold $\rho\in(0,1)$ and base accuracy $a_0\in[0,1]$, the decoder is $\rho$-faithful if
\begin{equation}
 \Prb(\widehat{y}=y^\star\mid \alpha_{\star,N}<\rho)\leq a_0.
\end{equation}
Here $a_0$ may be the chance accuracy of a multiple-choice task, the accuracy of a prior-only model, or the accuracy of an abstaining verifier.
\end{definition}

\begin{theorem}[Extreme-value attention poisoning]
\label{thm:poisoning}
Fix temperature $\tau>0$, threshold $\rho\in(0,1)$, base accuracy $a_0\in[0,1]$, and noise scale $\sigma>0$. For each context length $N$, let the context contain one decisive item with attention logit $S_\star\leq \gamma$ almost surely and $N$ effective distractors with independent logits
\begin{equation}
 Z_1,\ldots,Z_N \overset{\rm iid}{\sim} \mathcal{N}(0,\sigma^2).
\end{equation}
Let
\begin{equation}
 \alpha_{\star,N}=\frac{\exp(S_\star/\tau)}{\exp(S_\star/\tau)+\sum_{i=1}^N\exp(Z_i/\tau)}
\end{equation}
be the relevant attention mass, and assume the decoder is $\rho$-faithful. Then the expected accuracy $A_N=\Prb(\widehat{y}=y^\star)$ satisfies
\begin{equation}
\label{eq:main-bound}
 A_N\leq a_0+(1-a_0)\left[\Phi\left(\tfrac{\gamma+\tau\log\tfrac{1-\rho}{\rho}}{\sigma}\right)\right]^N,
\end{equation}
where $\Phi$ is the standard normal cdf. Consequently, if $\gamma,\tau,\rho,\sigma$ are fixed and $a_0$ is fixed, then $\limsup_{N\rightarrow\infty}A_N\leq a_0$. Conversely, for the upper bound to permit $A_N\geq 1-\eps$ with $\eps<1-a_0$, it is necessary that
\begin{equation}
\label{eq:margin-needed}
\begin{aligned}
 \gamma&\geq \sigma\Phi^{-1}\left(\left(\frac{1-\eps-a_0}{1-a_0}\right)^{1/N}\right)-\tau\log\frac{1-\rho}{\rho}\\
 &=\sigma\sqrt{2\log N}+O(1)
\end{aligned}
\end{equation}
for fixed $\eps,a_0,\rho,\tau,\sigma$ as $N\to\infty$.
\end{theorem}

\paragraph{Interpretation.} Equation~\eqref{eq:main-bound} is a formal poisoning curve. It says that the model's usable memory decays as the chance of at least one high-scoring distractor increases. The theorem also explains why ``just attend over more tokens'' can fail: unless architecture or training increases the relevant margin with $N$, the attention mass on the relevant evidence eventually becomes too small. The dependence is on $N_{\rm eff}$, not necessarily raw token length. A perfect gate that reduces the candidate set from $N$ to $K$ changes the bound by replacing $N$ with $K$.

\begin{corollary}[Random relevant logit]
Suppose $S_\star$ is independent of $Z_1,\ldots,Z_N$ and
$\Prb(S_\star>\gamma)\leq\delta$. Then
\begin{equation}
 A_N\leq \delta + a_0+(1-a_0)
 \left[
 \Phi\left(
 \frac{\gamma+\tau\log\frac{1-\rho}{\rho}}{\sigma}
 \right)
 \right]^N.
\end{equation}
Thus occasional high relevant logits help only through the tail
probability $\delta$; robust long-context use requires systematic
margin growth or a smaller effective distractor set.
\end{corollary}

\section{Mechanisms predicted by the theorem}

\paragraph{Softmax dilution.} Even when no single distractor wins, the denominator $\sum_i\exp(Z_i/\tau)$ can grow. The theorem uses a maximum event because it yields a clean lower bound on failure; denominator growth gives a second, cumulative failure route. Smaller temperature $\tau$ sharpens attention, but it also increases sensitivity to logit noise and can over-commit to an aliased distractor.

\paragraph{Position aliasing.} Long-context extrapolation requires distinguishing positions never or rarely seen during training. Methods such as ALiBi and RoPE encode useful inductive biases \citep{press2022train,su2024roformer}, while interpolation and scaling methods extend the range of positional features \citep{chen2023positionalinterpolation,peng2023yarn}. The definition above highlights a failure case: if two far-apart positions induce similar query-key contributions, a distractor can become indistinguishable from evidence for a given head. ``Found in the middle'' style interventions can be interpreted as changing positional logits to reduce the number of near-ties around relevant evidence \citep{zhang2024found}.

\paragraph{Retrieval versus use.} The result is compatible with work
showing that retrieval can be imperfect
\citep{liu2024lost,hsieh2024ruler}, but it also covers the stronger case
in which retrieval succeeds while problem solving still degrades
\citep{du2025context}. In the theorem, the evidence item is present,
but its score advantage is bounded; failure arises as the growing
candidate set increases the maximum competing score. In a multi-step
reasoning problem, similar interference can occur at multiple evidence
selection stages.

\section{Architectural alternatives}

The bound suggests a design principle: long-context systems should not merely increase the raw token window; they should control $N_{\rm eff}$, amplify the evidence margin, and verify the evidence before answer generation. Figure~\ref{fig:architecture} shows a candidate architecture.

\begin{figure*}[t]
\centering
\resizebox{.98\linewidth}{!}{%
\begin{tikzpicture}[node distance=1.00cm and .90cm, >=Latex,
  every node/.style={font=\small},
  block/.style={draw,rounded corners=2pt,align=center,minimum height=.78cm,minimum width=1.75cm},
  wide/.style={draw,rounded corners=2pt,align=center,minimum height=.78cm,minimum width=2.25cm}]
  \node[wide] (input) {long context\\$C_N$};
  \node[block,right=of input] (seg) {segmenter\\$B_1,\ldots,B_m$};
  \node[wide,right=of seg] (gate) {evidence gate\\$G_K(q,C_N)$};
  \node[wide,right=of gate] (cells) {$K$ evidence\\cells};
  \node[wide,right=of cells] (ver) {verifier and\\compressor};
  \node[wide,right=of ver] (dec) {short-context\\decoder};
  \node[block,below=1.05cm of gate] (pos) {alias-resistant\\position code};
  \node[block,below=1.05cm of cells] (loss) {contrastive\\anti-poison loss};
  \node[block,below=1.05cm of ver] (audit) {citation and\\consistency audit};
  \draw[->] (input) -- (seg);
  \draw[->] (seg) -- (gate);
  \draw[->] (gate) -- (cells);
  \draw[->] (cells) -- (ver);
  \draw[->] (ver) -- (dec);
  \draw[->] (dec) -- ++(1.0,0) node[right] {$\widehat{y}$};
  \draw[->,dashed] (pos) -- (gate);
  \draw[->,dashed] (loss) -- (cells);
  \draw[->,dashed] (audit) -- (ver);
  \node[font=\scriptsize,above=.30cm of gate] (ng) {$O(N)$ candidates};
  \node[font=\scriptsize,above=.30cm of cells] (kc) {$K\ll N$ retained};
  \draw[decorate,decoration={brace,amplitude=4pt},yshift=10pt]
    ($(gate.north west)+(-.12,.24)$) -- ($(cells.north east)+(.12,.24)$)
    node[midway,above=10pt,font=\scriptsize] {control $N_{\rm eff}$ and increase margin};
\end{tikzpicture}%
}
\caption{Poison-resilient long-context architecture. A gate converts a raw $O(N)$ context into $K$ evidence cells, alias-resistant positional coding reduces near-ties, contrastive losses push distractors below evidence, and a verifier forces the decoder to reason over recited, auditable evidence rather than the full noisy context.}
\label{fig:architecture}
\end{figure*}
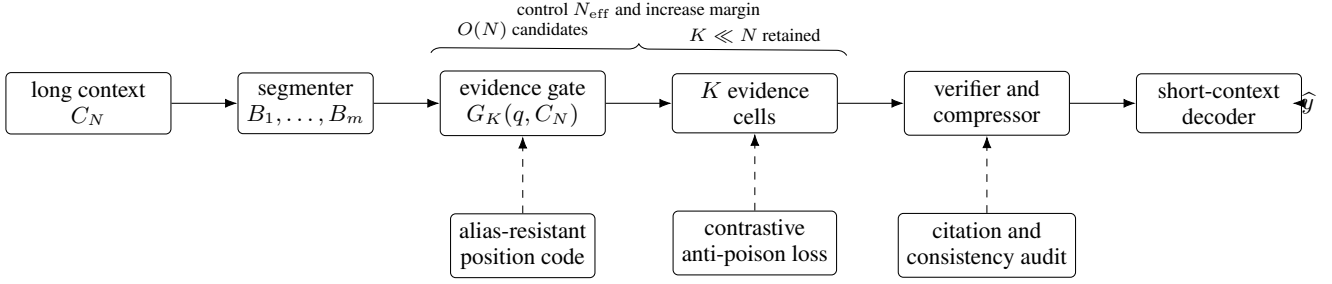

\subsection{Evidence bottleneck and retrieve-then-reason}

A model-agnostic mitigation is to split long-context processing into two calls or two internal phases: retrieve and recite the evidence, then solve from the shorter evidence prompt. This resembles retrieval-augmented generation \citep{lewis2020rag,guu2020realm,izacard2021leveraging}, but the retrieval source may be the user's own context rather than an external corpus. The theorem predicts that, conditional on retaining the decisive
evidence, the extreme-value term changes from $\Phi(t)^N$ to
$\Phi(t)^K$. Evidence omissions introduce a separate penalty governed
by the gate's miss probability $\eta$. The intervention is therefore
useful only when reducing the candidate set does not substantially
reduce evidence recall.

\begin{proposition}[Gate substitution]
\label{prop:gate}
Assume the theorem's conditions after a gate $G_K$ returns a set of at most $K$ distractors and includes the decisive evidence with probability at least $1-\eta$. Then
\begin{equation}
 A_N\leq \eta + a_0+(1-a_0)\left[\Phi\left(\frac{\gamma+\tau\log\frac{1-\rho}{\rho}}{\sigma}\right)\right]^K.
\end{equation}
The gate improves the bound when the recall penalty $\eta$ is smaller than the reduction in extreme-value risk from replacing $N$ by $K$.
\end{proposition}

\subsection{Alias-resistant positional and memory design}

Position encodings should be evaluated by their worst-case confusability, not only by perplexity. One can regularize heads so that for sampled query states $q$ and relevant positions $r$,
\begin{equation}
 \mathcal{L}_{\rm alias}=\E\left[\log\sum_{i\in D}\exp\left(\frac{s_i-s_r+m(i,r)}{\tau_a}\right)\right],
\end{equation}
where $m(i,r)$ is a distance-aware margin. This loss penalizes the log-sum-exp of distractor gaps and directly targets the denominator and maximum terms in the theorem. For streaming applications, preserving a small set of attention sinks can stabilize decoding \citep{xiao2024streamingllm}; however, sinks should be separated from semantic evidence cells so that sink tokens do not become spurious evidence.

\subsection{Verifier-mediated memory}

A verifier can transform long-context use from implicit attention to explicit evidence checking. Given candidate cells $E_1,\ldots,E_K$, the verifier produces a structured state
\begin{equation}
 z=(\text{claims},\text{support spans},\text{conflicts},\text{abstain flag}).
\end{equation}
The decoder receives $z$ and a compact evidence prompt. This creates two audit points: whether the key evidence was selected, and whether the answer follows from selected evidence. It also enables selective abstention when the evidence margin is below a calibrated threshold.

\subsection{Training objective}

For synthetic and real long-context tasks, construct triples $(q,x_\star,D,y^\star)$ with hard negatives. Let $s_\star$ be the score for the decisive evidence and $s_i$ for distractors. A direct anti-poison objective is
\begin{equation}
\begin{aligned}
 \mathcal{L}={}&\mathcal{L}_{\rm ans}-\lambda\log\frac{\exp(s_\star/\tau)}{\exp(s_\star/\tau)+\sum_{i\in D}\exp(s_i/\tau)}\\
 &+\mu\sum_{i\in D}\max\{0,s_i-s_\star+m_i\},
\end{aligned}
\end{equation}
where $m_i$ increases for semantically close or positionally confusable distractors. The second term controls attention mass; the third controls the maximum distractor gap. Prompt-compression methods such as LongLLMLingua are complementary when they preserve decisive evidence and remove low-utility tokens \citep{jiang2023longllmlingua}. Table~\ref{tab:solutions} summarizes these mitigation mechanisms and
their corresponding effects on the theoretical bound.

\begin{table}[t]
\caption{Mitigation mechanisms and their mathematical target.}
\label{tab:solutions}
\centering
\small
\begin{tabular}{p{0.22\linewidth}p{0.33\linewidth}p{0.32\linewidth}}
\toprule
Mechanism & Bound-level effect & Implementation sketch \\
\midrule
Evidence bottleneck & Replace $N$ by $K\ll N$ while paying recall error $\eta$ & Retrieve, recite, then solve; cite spans \\
Alias-resistant positions & Reduce probability that $s_i\approx s_\star$ & RoPE/ALiBi scaling plus margin regularization \\
Contrastive anti-poison loss & Increase $\gamma$ and suppress $\max_i Z_i$ & Hard negatives, shuffled positions, adversarial distractors \\
Verifier-mediated memory & Reduce $a_0$ failures and detect low-margin cases & Evidence table, consistency checks, abstention \\
Hierarchical memory & Bound per-level $N_{\rm eff}$ & Chunk summaries with back-pointers to raw spans \\
\bottomrule
\end{tabular}
\end{table}

\section{Experimental setup and metrics}
\label{sec:evaluation-protocol}

The theory recommends measuring the whole poisoning curve, not just the
maximum accepted length. A benchmark instance should preserve the
short-context answer while independently varying raw length $N$,
effective distractor count $N_{\rm eff}$, relevant position, number of
evidence items, semantic similarity of distractors, and positional
regime. Existing suites such as LongBench, L-Eval, ZeroSCROLLS, and
RULER provide useful starting points
\citep{bai2024longbench,an2024leval,shaham2023zeroscrolls,hsieh2024ruler}.

We report four quantities where available: (i) short-context accuracy
$A(N_0)$, (ii) the poisoning curve $P(N)=A(N_0)-A(N)$, (iii) evidence
recall, and (iv) conditional answer accuracy after successful retrieval.
The last quantity separates retrieval failure from evidence-use failure
\citep{du2025context}. Paired conditions reuse the same questions, and
failure outputs are inspected when possible to distinguish supplied-decoy
confusion from hallucination or formatting error.

\begin{table*}[t]
\centering
\small
\setlength{\tabcolsep}{4pt}
\caption{Overview of the controlled studies. The first two studies are
reported in the main paper; the gate intervention is reported in the
appendix. The designs vary complementary axes rather than repeating the
same manipulation.}
\label{tab:study-overview}
\begin{tabular}{@{}p{0.19\textwidth}p{0.17\textwidth}p{0.20\textwidth}p{0.18\textwidth}p{0.20\textwidth}@{}}
\toprule
Study & Model(s) & Primary intervention & Context regime & Diagnostic separation \\
\midrule
Hard-negative decay
& GPT-4o
& Increase distractor count while retaining embedded decoys
& $N\in\{100,300,500,1000\}$ sentences
& Manual audit verifies that errors select supplied decoy values \\
Distractor confusability
& Claude Sonnet~4, GPT-4.1, Gemma~3 (12B)
& Change topical, structural, and entity-name overlap at fixed $N$
& $300$ records, approximately $20$--$21$K tokens
& Answer accuracy and cited-record recall distinguish extraction from retrieval \\
Retrieve-then-reason gate
& GPT-4.1
& Replace all $N$ chunks with BM25 top-$K$ candidates
& $4$K--$512$K tokens
& Recall@$K$ and accuracy conditional on retrieving all support \\
\bottomrule
\end{tabular}
\end{table*}

Table~\ref{tab:study-overview} divides the hypothesis space across
candidate count, distractor composition, and mitigation. We summarize
controlled study properties rather than model architecture because the
hosted systems do not expose comparable internal specifications. No
single study validates the full theory; together they test distinct,
non-redundant implications.

\section{Empirical study: retrieval decay with embedded hard negatives}
\label{sec:decay}

\nameref{sec:evaluation-protocol} introduces a minimal synthetic test
of Theorem~\ref{thm:poisoning}---key-value retrieval with a variable number of hard negatives. We carry out that test here, checking whether the theorem's central prediction, that accuracy decays toward a chance floor as the number of distractors grows, holds for a modern production model.

\paragraph{Setup.} We embed a single target fact in a passage of $N$ surrounding sentences and ask GPT-4o to retrieve it. The filler is drawn from a public-domain novel, so it reads as ordinary prose rather than an artificial list. Critically, the target fact is not the only sentence of its kind: several \emph{decoy} sentences, each stating a different value in the same natural phrasing as the true fact, are woven into the passage. This makes the task a genuine test of retrieval under hard negatives rather than trivial pattern matching---the model must identify \emph{which} value is being asked about from narrative context, not merely notice that a value is present. The target's position is randomized on each trial, and $N$ is varied across $\{100,300,500,1000\}$ with $10$ independent trials per condition.

\paragraph{Prediction.} Theorem~\ref{thm:poisoning} predicts that retrieval accuracy decays toward a chance floor $a_0$ as $N$ grows, following
\begin{equation}
 A(N)\approx a_0+(1-a_0)\,q^{N},\qquad q=\Phi(b/\sigma)\in(0,1),
 \label{eq:decay-fit}
\end{equation}
unless the target's evidence margin grows with $N$. We fit Eq.~\eqref{eq:decay-fit} to the observed accuracy by nonlinear least squares, and we inspected every incorrect response individually to confirm that failures reflect genuine confusion among competing candidates (the model's answer was always one of the decoy values actually present in that trial's passage) rather than hallucination or task misunderstanding.

\begin{table}[t]
    \centering
    \caption{Retrieval accuracy vs.\ number of distractors $N$, GPT-4o, 10 trials per condition.}
    \label{tab:empirical-decay}
    \begin{tabular}{@{}lccc@{}}
        \toprule
        $N$ & Trials & Accuracy & Mean prompt length (chars) \\
        \midrule
        100  & 10 & 100\% & 8{,}952  \\
        300  & 10 & 80\%  & 25{,}978 \\
        500  & 10 & 80\%  & 41{,}954 \\
        1000 & 10 & 80\%  & 84{,}668 \\
        \bottomrule
    \end{tabular}
\end{table}

\paragraph{Results.} Table~\ref{tab:empirical-decay} shows accuracy holding at $100\%$ at $N=100$ before dropping to a stable $80\%$ by $N=300$ and remaining there through $N=1000$. Fitting Eq.~\eqref{eq:decay-fit} gives $\hat{a}_0 = 0.777$, $\hat{q} = 0.9962$, with $R^2 = 0.74$.

\paragraph{Interpretation.}
These results provide initial behavioral evidence consistent with the
theorem's qualitative prediction: retrieval accuracy decreases as
exposure to hard-negative distractors grows. Every observed error
selected a decoy value present in the prompt, ruling out hallucination,
refusal, or task misunderstanding as the immediate failure type.
Because the experiment does not observe internal attention scores, it
does not establish attention interference as the unique underlying
mechanism. The fitted exponential form is also exploratory rather than
conclusive ($R^2=0.74$): accuracy drops between $N=100$ and $N=300$
and then plateaus through $N=1000$, suggesting a threshold-like
transition over the tested range. A denser sweep around this transition
and more trials per condition are needed to estimate the curve's shape
reliably.

This study varies candidate exposure but cannot separate raw count from
candidate composition. The next study holds record count and approximate
token length fixed while changing the form of competition.

\section{Empirical study: distractor confusability at fixed context length}
\label{sec:ablation}

We instantiate the evaluation protocol with a controlled study that holds raw context length fixed while varying observable properties of distractor confusability. This design tests a behavioral implication of score aliasing (Definition~\ref{def:aliasing}): retrieval performance may depend not only on the number of tokens or records in the context, but also on how strongly the distractors compete with the gold evidence. Because the experiment does not directly observe attention scores, it should be interpreted as a behavioral test of this implication rather than a direct measurement of the effective distractor count $N_{\rm eff}$.

\paragraph{Setup.} Each instance is a key--value retrieval query over $N=300$ records (approximately $20{,}000$--$21{,}000$ tokens after length matching) with exactly one gold record; the model must return the answer and cite the evidence record ID. We evaluate three models from different model families and deployment settings---Claude Sonnet~4, GPT-4.1, and Gemma~3 (12B)---at temperature zero. Holding $N$ fixed, we compare four distractor constructions: unrelated filler, same-domain/different-attribute records, same-format/different-entity records, and near-alias records whose entities share a surname with the gold entity. These conditions separately manipulate topical overlap, structural overlap, and entity-name overlap; they should not be interpreted as points on a single monotonic similarity scale. Records are padded to a fixed length so the manipulation is not confounded with token count, and evidence position is balanced across beginning, middle, and end. The full protocol, prompts, and per-model tables are provided in the appendix (\nameref{app:ablation}).

\paragraph{Results.} Table~\ref{tab:ablation-main} and Figure~\ref{fig:ablation} summarize the outcome. Accuracy is near-ceiling for unrelated and same-domain distractors for all three models, then decreases in the same-format condition: Claude falls to $0.889$ (an $11.1$ percentage-point paired drop against baseline, $95\%$ CI $[+2.8,+22.2]$---the only statistically significant paired contrast in the study), GPT-4.1 falls to $0.967$, and Gemma falls to approximately $0.85$. The character-trigram Jaccard measure serves specifically as an entity-name lexical-overlap check. It remains below $0.01$ for the first three conditions and rises sharply to $0.317$ only in the near-alias condition (Figure~\ref{fig:ablation}, right). Thus, the decrease in the same-format condition occurs despite low entity-name overlap and cannot be explained by name similarity alone.

\begin{table*}[t]
\caption{Answer accuracy by model and distractor construction at fixed context length ($N=300$ records, $\sim$20K tokens, errors excluded). The rightmost column reports entity-name lexical overlap, measured by character-trigram Jaccard similarity to the gold entity. Gemma values are approximate due to a high endpoint error rate. Full results with evidence recall, valid counts, and paired contrasts are provided in the appendix (\nameref{app:ablation}).}
\label{tab:ablation-main}
\centering
\small
\begin{tabular}{lcccc}
\toprule
Distractor condition & Claude Sonnet 4 & GPT-4.1 & Gemma 3 (12B) & Sim.\ proxy \\
\midrule
Unrelated filler & 1.000 & 1.000 & 1.000 & 0.004 \\
Same domain, different attribute & 1.000 & 1.000 & 1.000 & 0.005 \\
Same format, different entity & \textbf{0.889} & \textbf{0.967} & $\sim$\textbf{0.85} & 0.010 \\
Near-alias entity & 1.000 & 0.983 & $\sim$0.89 & 0.317 \\
\bottomrule
\end{tabular}
\end{table*}

\begin{figure}[t]
\centering
\includegraphics[width=\linewidth]{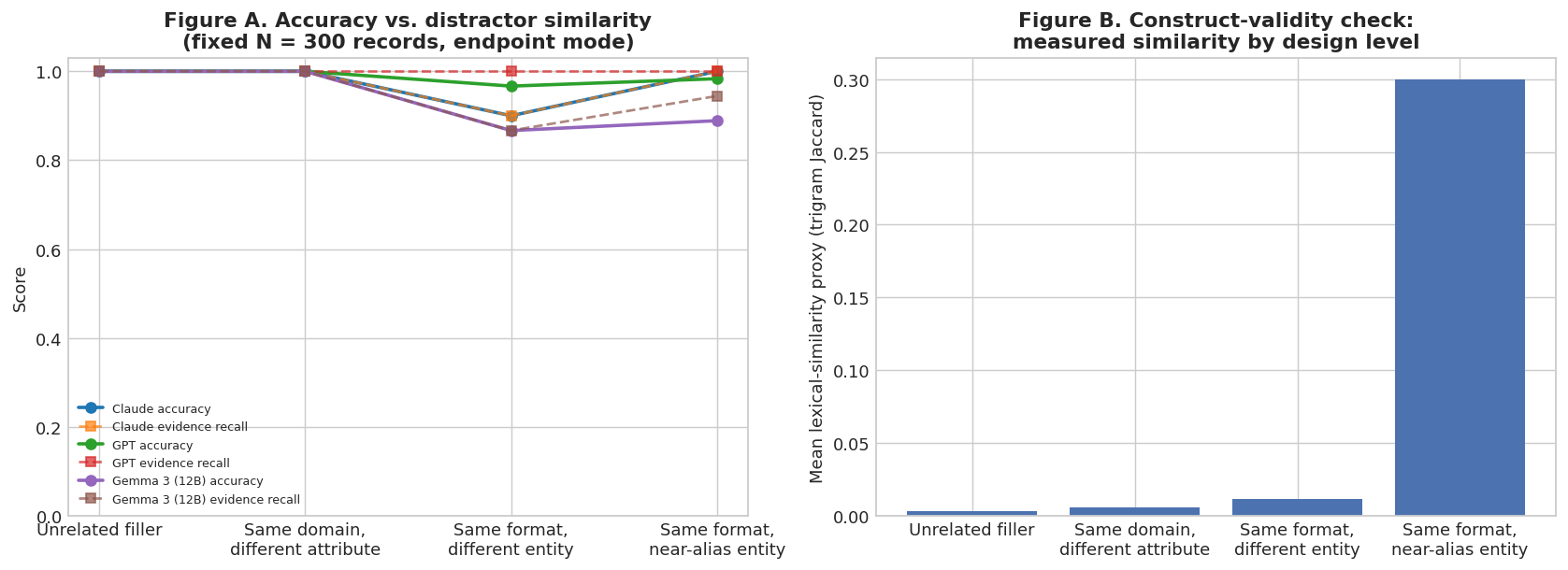}
\caption{Distractor-confusability ablation at fixed context length ($N=300$). \textbf{Left:} answer accuracy and evidence recall across four distractor constructions; all models show their largest decrease in the same-format condition, while the two frontier models recover toward baseline in the near-alias condition. \textbf{Right:} entity-name lexical-overlap check---the character-trigram Jaccard measure is near zero for the first three conditions and rises only in the near-alias condition. The accuracy decrease in the same-format condition therefore cannot be attributed to high entity-name overlap.} 
\label{fig:ablation}
\end{figure}

\paragraph{Two findings.} First, \emph{structural overlap is more damaging in these data than entity-name overlap}. The largest observed degradation for every model occurs in the same-format condition, where all $300$ records share the gold record's template but refer to different entities. This pattern is behaviorally consistent with increased score aliasing, because a shared structure may create a larger set of competitive candidates. However, the experiment does not directly observe attention scores or determine which distractors fall within an $\eps$-aliasing band. Second, we observe a \emph{near-alias paradox}: despite substantially greater entity-name similarity, accuracy in the near-alias condition recovers toward baseline for the two frontier models (Claude $1.000$, GPT-4.1 $0.983$). One possible explanation is that the first name remains a discriminative cue, whereas structurally identical records provide fewer template-level features for distinguishing the gold record. This explanation is a hypothesis rather than a directly tested mechanism. Gemma shows the largest drops overall, but its endpoint rejected $96\%$ of these $\sim$20K-token calls, so its estimates are underpowered and should be interpreted as suggestive.

\paragraph{Interpretation and cross-study synthesis.}
The study supports the narrower conclusion that distractor composition
matters even when record count and approximate token length are held
constant. The same-format condition produces the largest observed
accuracy decrease, whereas the condition with the greatest entity-name
overlap does not. This pattern is consistent with the distinction
between raw context size and effective distractor count, although the
experiment does not directly measure $N_{\rm eff}$ or the underlying
attention-score geometry.

Under the proposed interpretation, structurally identical records may
increase the number of candidates that compete strongly with the gold
record, while the near-alias condition may preserve discriminative
entity-level cues. The retrieval-decay study complements this result by
varying the number of exposed hard negatives rather than their
construction. Together, the two studies indicate that neither raw token
count nor lexical similarity alone explains long-context failure;
candidate exposure and structural confusability both matter. This
synthesis motivates the appendix gate experiment, which tests whether
reducing the competitive candidate set improves evidence use once
retrieval recall is measured separately. The findings also motivate
defenses such as record-format diversification, explicit provenance
metadata, and retrieval mechanisms that penalize suspiciously
repetitive templates.

\section{Limitations and conclusion}

The theorem isolates one long-context failure mechanism rather than
modeling an entire LLM. It assumes independent Gaussian distractor
logits, a single decisive item, and a threshold model of evidence use.
Real systems contain correlated representations, multiple heads,
residual pathways, memorized priors, tools, and instruction hierarchies.
The result should therefore be interpreted as a mechanistic diagnostic,
not as a calibrated predictor for a specific model. The appendix
attention analyses test these assumptions directly and show that,
although the Gaussian i.i.d.\ abstraction is not exact, the predicted
extreme-value scale remains a useful reference.

The empirical evidence is also deliberately controlled. The
retrieval-decay study uses only $10$ trials per condition, while the
fixed-length ablation measures behavioral proxies for confusability
rather than attention-score margins directly. Unequal endpoint
reliability also makes the Gemma estimates suggestive. Both main-paper
studies emphasize retrieval-style tasks, so code repositories,
conversations, and temporally conflicting documents may exhibit
additional failure modes.

Within these limits, the paper establishes a consistent result across
theory and experiment. A fixed evidence margin becomes increasingly
fragile as the effective distractor set grows; retrieval performance
depends on both candidate exposure and distractor structure; and
reducing the competitive set helps only when required evidence is
preserved. Future work should therefore cross context length with
distractor construction, connect internal score margins to end-to-end
errors, and compare fixed with adaptive retrieval budgets. Gate studies
should report final accuracy, all-support recall, and conditional
accuracy together, because any one metric can conceal failure in
another stage.

The central conclusion is that a larger context window increases both
available information and the opportunities for distractors to compete
with decisive evidence. The relevant quantity is the effective number
and strength of confusable tokens or chunks, not raw window length
alone. Robust long-context systems should therefore combine evidence
selection, alias control, recall-preserving retrieval, and explicit
verification. A context window is capacity; usable memory is capacity
plus reliable addressing.

\clearpage
\bibliography{aaai2027}

\clearpage
\appendix
\input{appendix}

\end{document}

%% file: appendix.tex

\section{Proofs}
\label{app:proofs}

\begin{proof}[Proof of Theorem~\ref{thm:poisoning}]
Define
\begin{equation}
 b=\gamma+\tau\log\frac{1-\rho}{\rho},\qquad M_N=\max_{1\leq i\leq N}Z_i.
\end{equation}
If $M_N\geq b$, then there exists a distractor $j$ such that
\begin{equation}
 Z_j-S_\star\geq Z_j-\gamma\geq \tau\log\frac{1-\rho}{\rho}.
\end{equation}
Using only this distractor in the softmax denominator gives
\begin{align}
 \alpha_{\star,N}
 &=\frac{\exp(S_\star/\tau)}{\exp(S_\star/\tau)+\sum_{i=1}^N\exp(Z_i/\tau)}\\
 &\leq \frac{\exp(S_\star/\tau)}{\exp(S_\star/\tau)+\exp(Z_j/\tau)}
 =\frac{1}{1+\exp((Z_j-S_\star)/\tau)}\\
 &\leq \frac{1}{1+(1-\rho)/\rho}=\rho.
\end{align}
Therefore, if $M_N>b$, then $\alpha_{\star,N}<\rho$, and hence
\[
\{\alpha_{\star,N}\geq\rho\}\subseteq\{M_N\leq b\}.
\]
Since the $Z_i$ are i.i.d.\ Gaussian, $M_N$ has a continuous
distribution, so
\begin{equation}
\Prb(M_N\leq b)
=
\prod_{i=1}^{N}\Prb(Z_i\leq b)
=
\left[\Phi\left(\frac{b}{\sigma}\right)\right]^N.
\end{equation}
Thus
\begin{equation}
\Prb(\alpha_{\star,N}\geq\rho)
\leq
\Prb(M_N\leq b)
=
\left[\Phi\left(\frac{b}{\sigma}\right)\right]^N.
\end{equation}
By $\rho$-faithfulness,
\begin{align}
 A_N
 &=\Prb(\widehat{y}=y^\star,\alpha_{\star,N}\geq\rho)+\Prb(\widehat{y}=y^\star,\alpha_{\star,N}<\rho)\\
 &\leq \Prb(\alpha_{\star,N}\geq\rho)+a_0\Prb(\alpha_{\star,N}<\rho)\\
 &= a_0+(1-a_0)\Prb(\alpha_{\star,N}\geq\rho)\\
 &\leq a_0+(1-a_0)\left[\Phi\left(\frac{\gamma+\tau\log\frac{1-\rho}{\rho}}{\sigma}\right)\right]^N.
\end{align}
This proves the finite-sample accuracy bound stated in Theorem~\ref{thm:poisoning}. If $b$ is fixed, then $\Phi(b/\sigma)<1$ unless $b=+\infty$, so the exponential term tends to zero and $\limsup_{N\to\infty} A_N\leq a_0$.

For the bound to permit accuracy at least $1-\eps$, it is necessary that
\begin{equation}
a_0+(1-a_0)\Phi(b/\sigma)^N\geq 1-\eps.
\end{equation}
Rearranging gives
\begin{equation}
\Phi(b/\sigma)\geq
\left(\frac{1-\eps-a_0}{1-a_0}\right)^{1/N}.
\end{equation}
Because $\Phi^{-1}$ is increasing,
\begin{equation}
b
\geq
\sigma\Phi^{-1}\left(
\left(\frac{1-\eps-a_0}{1-a_0}\right)^{1/N}
\right).
\end{equation}
Substituting
\[
b=\gamma+\tau\log\frac{1-\rho}{\rho}
\]
yields the finite-sample margin requirement stated in Theorem~\ref{thm:poisoning}.

To obtain its asymptotic form, define
\[
r=\frac{1-\eps-a_0}{1-a_0}\in(0,1)
\qquad\text{and}\qquad
\kappa=-\log r>0.
\]
Then
\begin{equation}
r^{1/N}
=
\exp\left(-\frac{\kappa}{N}\right)
=
1-\frac{\kappa}{N}+O(N^{-2}).
\end{equation}
Using the standard Gaussian upper-quantile asymptotic,
\begin{equation}
\Phi^{-1}\left(1-\frac{\kappa}{N}+O(N^{-2})\right)
=
\sqrt{2\log N}+O(1).
\end{equation}
Therefore,
\begin{equation}
\gamma
\geq
\sigma\sqrt{2\log N}+O(1),
\end{equation}
which proves the stated asymptotic margin requirement.
\end{proof}

\begin{proof}[Proof of the corollary]
Let $E=\{S_\star\leq\gamma\}$. On $E$, the theorem's proof applies conditionally. On $E^c$, bound accuracy by one. Hence
\begin{align}
A_N&\leq \Prb(E^c)\notag\\
&\quad+\Prb(E)\left(a_0+(1-a_0)\left[\Phi\left(\tfrac{\gamma+\tau\log\tfrac{1-\rho}{\rho}}{\sigma}\right)\right]^N\right)\\
&\leq \delta+a_0+(1-a_0)\left[\Phi\left(\tfrac{\gamma+\tau\log\tfrac{1-\rho}{\rho}}{\sigma}\right)\right]^N.
\end{align}
\end{proof}

\section{Additional implementation details}
\label{app:implementation}

A practical implementation of the retrieve-then-reason architecture
shown in Figure~\ref{fig:architecture} of the main paper can be
built without changing a base LLM: chunk the context, retrieve top-$K$ chunks with a lexical and dense hybrid scorer \citep{robertson2009probabilistic}, ask the model to quote relevant spans, and start a new short-context call that contains only the quoted spans and the question. A model-internal implementation can train a gate and verifier jointly with the decoder. In both cases, the evaluation should report gate recall separately from final answer accuracy.

\section{Distractor confusability ablation: full protocol and results}
\label{app:ablation}
This section provides the complete protocol and supplemental tables for the distractor-confusability study reported in the main paper (\nameref{sec:ablation}).
\subsection{Task and distractor constructions}
Each instance poses a factual question about a fictional character drawn from well-known literary universes (Harry Potter, The Lord of the Rings, The Chronicles of Narnia, A Song of Ice and Fire); the entities are used only to construct newly written, factual-style records, and no copyrighted text is reproduced. The model receives one gold record and $299$ distractor records, for a total of $N_{\mathrm{records}}=300$ records, and must return a JSON object containing the extracted answer and the cited evidence record ID. Illustrated for the question ``What is Stannis Baratheon's castle?'' (gold answer: Dragonstone; gold record ID: gold-0000), the four distractor conditions are:

\begin{itemize}
\item \textbf{Unrelated filler (baseline).} Distractors are catalogue notes about unrelated objects (e.g.\ ``a dragon egg''), with no lexical or semantic overlap with the gold entity or attribute. The gold record is the only topically relevant record.
\item \textbf{Same domain, different attribute.} Distractors are facts from the same fictional universe but about different attributes and entity types (lineages, battles, geography), creating topical but non-responsive overlap.
\item \textbf{Same format, different entity.} Every distractor uses the gold record's exact template (``Authoritative record: X's castle is Y'') for a clearly different character, so the model must discriminate purely on entity name among $300$ structurally identical records.
\item \textbf{Near-alias entity.} Distractors use the same template and name entities that share the gold entity's surname (e.g.\ Robert, Renly, and Joffrey Baratheon versus Stannis Baratheon), each giving a factually correct but wrong-for-the-question answer, so a model that matches on surname alone extracts the wrong answer.
\end{itemize}

\subsection{Controls}

The context contains $N_{\mathrm{records}}=300$ records in every condition. Each distractor record is padded to a fixed word count with neutral filler tokens so the similarity manipulation is not confounded with text length. Evidence position is balanced across beginning, middle, and end, and each base instance appears exactly once per similarity condition to enable within-instance paired comparisons. A model-independent lexical-similarity proxy (character-trigram Jaccard between each distractor entity name and the gold entity name) is computed per condition as a manipulation check. API errors (timeouts, rate limits, context-length rejections) are excluded from accuracy calculations; only successful responses are scored.

\subsection{Models and scale}

Claude Sonnet~4 and GPT-4.1 were accessed through an OAuth2 gateway; Gemma~3 (12B) through a pay-per-token foundation-model API. Claude and GPT used $60$ base instances $\times\,4$ levels $=240$ conditions each; Gemma used an expanded design of $460$ base instances ($1{,}840$ conditions) to accumulate valid calls despite a high endpoint error rate. Each call processes roughly $20{,}000$ tokens of context. Answer accuracy uses loose string normalization (lowercased, articles stripped, whitespace collapsed); evidence recall requires the correct gold record ID; paired accuracy drops are estimated by paired bootstrap ($5{,}000$ iterations, seed $123$).

\subsection{Construct validity and length matching}

The entity-name lexical-overlap proxy remains small in the unrelated, same-domain, and same-format conditions ($0.004$--$0.010$) and rises to $0.317$ in the near-alias condition. This verifies that the near-alias construction specifically increases entity-name overlap, whereas the same-format construction does not. Median approximate token counts are similar across conditions, so the observed accuracy differences are not explained by large changes in context length (Table~\ref{tab:construct}).

\begin{table*}[h]
\caption{Construct-validity and length-matching check: mean lexical-similarity proxy (character-trigram Jaccard between distractor and gold entity names) and median approximate token count by distractor condition, at $N_{\mathrm{records}}=300$.}
\label{tab:construct}
\centering
\small
\begin{tabular}{lcc}
\toprule
Distractor condition & Mean similarity proxy & Median tokens \\
\midrule
Unrelated filler & 0.004 & 20{,}508 \\
Same domain, different attribute & 0.005 & 20{,}149 \\
Same format, different entity & 0.010 & 21{,}055 \\
Near-alias entity & 0.317 & 21{,}053 \\
\bottomrule
\end{tabular}
\end{table*}

\subsection{API reliability}

At $\sim$20K tokens, endpoint reliability varied sharply across providers (Table~\ref{tab:reliability}). GPT-4.1 completed all calls; Claude's error rate scaled with token count (approximately $11.7\%$ at $\sim$9K, $35.4\%$ at $\sim$20K, $59.6\%$ at $\sim$35K in companion runs); Gemma's endpoint rejected most $\sim$20K-token inputs at a context-length limit. All accuracy analyses use valid (non-error) calls only.

\begin{table*}[h]
\caption{Endpoint reliability at $\sim$20K tokens. Accuracy analyses are computed on valid calls only.}
\label{tab:reliability}
\centering
\small
\begin{tabular}{lcccl}
\toprule
Model & Total calls & Valid calls & Error rate & Primary cause \\
\midrule
Claude Sonnet 4 & 240 & 155 & 35.4\% & Rate limits / timeouts \\
GPT-4.1 & 240 & 240 & 0.0\% & --- \\
Gemma 3 (12B) & 1{,}840 & 72 & 96.1\% & Context-length limit \\
\bottomrule
\end{tabular}
\end{table*}

\subsection{Full accuracy, evidence recall, and paired contrasts}

Table~\ref{tab:full-acc} reports answer accuracy and evidence recall
for every model and distractor condition with valid counts;
Table~\ref{tab:paired} reports paired accuracy drops relative to the
unrelated-filler baseline. Claude's $+0.111$ same-format drop is the
only paired contrast whose reported $95\%$ confidence interval
excludes zero. GPT-4.1 has a smaller point-estimate drop in the same
direction ($+0.033$), but its confidence interval includes zero.
Notably, GPT-4.1's evidence recall remains at $1.000$ across all
conditions: when it errs, it cites the correct record but extracts an
incorrect answer, indicating an evidence-use rather than a retrieval
failure.

\begin{table*}[h]
\caption{Answer accuracy and evidence recall by model and distractor condition ($N_{\mathrm{records}}=300$, $\sim$20K tokens, errors excluded). Gemma counts and values are approximate due to the $96.1\%$ error rate and should be interpreted with caution.}
\label{tab:full-acc}
\centering
\small
\begin{tabular}{llccc}
\toprule
Model & Distractor condition & Accuracy & Evidence recall & Valid $n$ \\
\midrule
Claude Sonnet 4 & Unrelated filler & 1.000 & 1.000 & 51 \\
Claude Sonnet 4 & Same domain, different attribute & 1.000 & 1.000 & 36 \\
Claude Sonnet 4 & Same format, different entity & 0.889 & 0.889 & 36 \\
Claude Sonnet 4 & Near-alias entity & 1.000 & 1.000 & 32 \\
\midrule
GPT-4.1 & Unrelated filler & 1.000 & 1.000 & 60 \\
GPT-4.1 & Same domain, different attribute & 1.000 & 1.000 & 60 \\
GPT-4.1 & Same format, different entity & 0.967 & 1.000 & 60 \\
GPT-4.1 & Near-alias entity & 0.983 & 1.000 & 60 \\
\midrule
Gemma 3 (12B) & Unrelated filler & 1.000 & 1.000 & $\sim$22 \\
Gemma 3 (12B) & Same domain, different attribute & 1.000 & 1.000 & $\sim$20 \\
Gemma 3 (12B) & Same format, different entity & $\sim$0.85 & $\sim$0.85 & $\sim$16 \\
Gemma 3 (12B) & Near-alias entity & $\sim$0.89 & $\sim$0.89 & $\sim$14 \\
\bottomrule
\end{tabular}
\end{table*}

\begin{table*}[h]
\caption{Paired accuracy drop relative to the unrelated-filler baseline (positive $=$ harder; paired bootstrap, $5{,}000$ iterations, errors excluded). Gemma is omitted: sparse matched pairs ($n=4$--$5$) preclude reliable paired testing.}
\label{tab:paired}
\centering
\small
\begin{tabular}{llccc}
\toprule
Model & Comparison condition & Accuracy drop $\Delta$ & 95\% CI & Significant? \\
\midrule
Claude Sonnet 4 & Same domain, different attribute & $0.000$ & $[0.000, 0.000]$ & No \\
Claude Sonnet 4 & Same format, different entity & $+0.111$ & $[+0.028, +0.222]$ & Yes \\
Claude Sonnet 4 & Near-alias entity & $0.000$ & $[0.000, 0.000]$ & No \\
\midrule
GPT-4.1 & Same domain, different attribute & $0.000$ & $[0.000, 0.000]$ & No \\
GPT-4.1 & Same format, different entity
& $+0.033$ & $[0.000,+0.083]$ & No \\
GPT-4.1 & Near-alias entity & $+0.017$ & $[0.000, +0.050]$ & No \\
\bottomrule
\end{tabular}
\end{table*}

Across the four condition-level means, the descriptive Spearman
correlations between the lexical-overlap proxy and mean accuracy are
$\rho=-0.258$ for Claude, $\rho=-0.738$ for GPT-4.1, and
$\rho=-0.738$ for Gemma. Because these coefficients are based on only
four condition-level observations and include tied accuracy values,
they are descriptive summaries rather than inferential evidence.

\subsection{Limitations and adversarial-RAG implications}

Several caveats qualify these results. Claude's $35.4\%$ error rate raises a survivorship concern---valid calls may under-represent the hardest instances---so the effect should be replicated on a higher-throughput endpoint. Gemma is underpowered ($72/1{,}840$ valid); its aggregate pattern is suggestive but not conclusive, and a dedicated run at shorter context ($N\approx 64$--$128$) would provide more reliable estimates. The near-alias manipulation tops out at Jaccard $0.317$, which frontier models handle well; exact-paraphrase or same-entity/different-date distractors could reveal stronger effects. Because the entities are drawn from popular fiction, pretraining memorization may make the task easier than novel-entity retrieval in production RAG, so the surprising same-format-over-near-alias ordering warrants replication with non-fictional entities. Finally, the present analysis evaluates the four distractor conditions at $N=300$ and therefore does not estimate an interaction between context length and distractor construction.
The observed same-format drop suggests a potential vulnerability:
injecting structurally similar records containing plausible alternative
answers may degrade accuracy even when the gold evidence remains
present. In this experiment, the same-format condition produced a
larger point-estimate drop than the near-alias condition, although that
difference was not tested directly. The results motivate defenses that
diversify record formats, attach provenance metadata, and de-rank
suspiciously repetitive templates. The present study does not estimate
how this effect changes with context length.

\section{Attention-distribution diagnostics and dose-response: from assumption to mechanism}
\label{app:attention}

Theorem~\ref{thm:poisoning}'s bound rests on a specific structural assumption: that distractor attention logits are independent, identically distributed Gaussian noise. This section tests that assumption directly on an open-weight model and then asks two follow-up questions that an internal diagnostic cannot answer on its own: does the assumption violation correspond to behavioral failure, and, if raw context length alone does not explain failure, what does? These findings complement, rather than replace, the empirical studies reported in the main paper (\nameref{sec:decay} and \nameref{sec:ablation}).

We use Gemma-3-4B-IT \citep{gemmateam2025gemma3} with \texttt{output\_attentions} enabled, run locally via Hugging Face \texttt{transformers}. Each trial builds a needle-in-haystack prompt: one evidence sentence stating a random secret code, embedded among $N$ distractor sentences, followed by a question asking for the code. For each of the $34$ layers and $8$ attention heads ($272$ layer-head pairs), we sum the final answer-generating token's post-softmax attention weights over the tokens belonging to each distractor sentence, producing a sentence-level attention-mass value. We analyze the logarithm of this quantity after robust median/MAD
standardization. Because sentence aggregation occurs before the
logarithm, this statistic is a sentence-level log-attention score rather
than a recovered raw attention logit. We sweep $N \in \{20, 50, 100\}$ with 10 trials each for the attention analysis (limited by the quadratic memory cost of storing full attention weights), and separately test purely behavioral accuracy (no attention extraction, hence much cheaper) at larger $N$ and under varying distractor content.

\begin{figure}[h]
    \centering
    \includegraphics[width=\columnwidth]{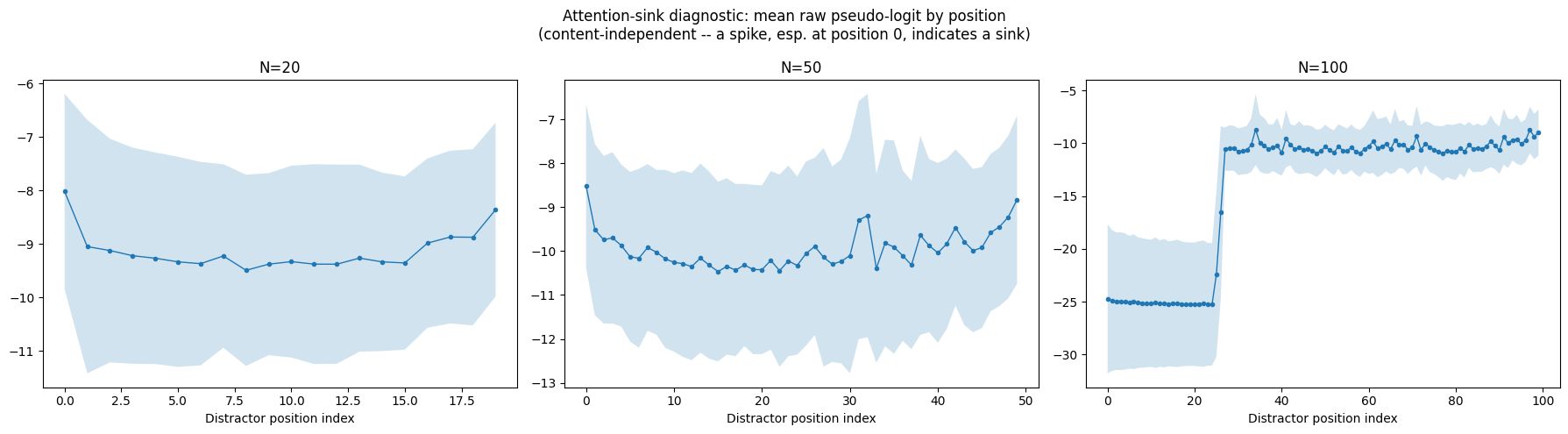}
    \caption{Mean attention by sentence position only, independent of content, pooled across trials/layers/heads.}
    \label{fig:sink}
\end{figure}

Figure~\ref{fig:sink} diagnoses a positional confound at every $N$ tested: attention depends measurably on position alone. At $N=20$ and $N=50$ this reproduces the U-shaped ``lost in the middle'' effect reported by \citet{liu2024lost}; at $N=100$ it instead shows a sharp step, consistent with Gemma-3's architecture interleaving sliding-window and global-attention layers.

\begin{table*}[h]
\centering
\caption{Attention distribution tests by $N$ (34 layers $\times$ 8 heads, 10 trials per $N$).}
\label{tab:attention_tests}
\small
\begin{tabular}{@{}lccc@{}}
\toprule
 & $N=20$ & $N=50$ & $N=100$ \\
\midrule
\multicolumn{4}{@{}l}{\textit{Gaussianity (marginal distribution)}} \\
Shapiro--Wilk $p$              & $4.0\times10^{-26}$ & $3.5\times10^{-33}$ & $<10^{-300}$ \\
Anderson--Darling stat.        & 205.9 & 1298.2 & 35768.7 \\
Excess kurtosis (95\% CI)      & 2.70 [2.35, 3.09] & 6.57 [6.04, 7.03] & 9.12 [8.83, 9.39] \\
\addlinespace
\multicolumn{4}{@{}l}{\textit{Independence (lags 1--10)}} \\
Ljung--Box $Q$                 & 3.74  & 11.50 & 490.60 \\
Ljung--Box $p$                 & 0.958 & 0.320 & $<10^{-99}$ \\
\addlinespace
\multicolumn{4}{@{}l}{\textit{Extreme value (max statistic vs.\ Gumbel)}} \\
KS statistic                   & 0.204 & 0.622 & 0.318 \\
KS $p$                         & $1.2\times10^{-99}$ & $<10^{-300}$ & $2.2\times10^{-244}$ \\
Empirical mean max (std.)      & 2.31 & 4.18 & 3.02 \\
Leading-order scale $\sqrt{2\log N}$    & 2.45 & 2.80 & 3.04 \\
\bottomrule
\end{tabular}
\end{table*}

The Gaussianity tests reject normality at every $N$, and the KS tests reject an exact Gumbel fit at every $N$ for Gemma-3 (Table~\ref{tab:attention_tests}). The Ljung--Box test does not reject independence at $N=20$ or $N=50$ ($p=0.958$ and $p=0.320$, respectively), but rejects it strongly at $N=100$ ($p<10^{-99}$). Thus, the Gaussian i.i.d.\ abstraction is not quantitatively exact, although the nature of the departure changes with context length. Heavy tails may increase extreme-score risk, whereas dependence can either amplify or suppress it. We therefore treat Theorem~\ref{thm:poisoning} as a diagnostic baseline rather than a calibrated quantitative bound for these models.

\begin{figure}[h]
    \centering
    \includegraphics[width=\columnwidth]{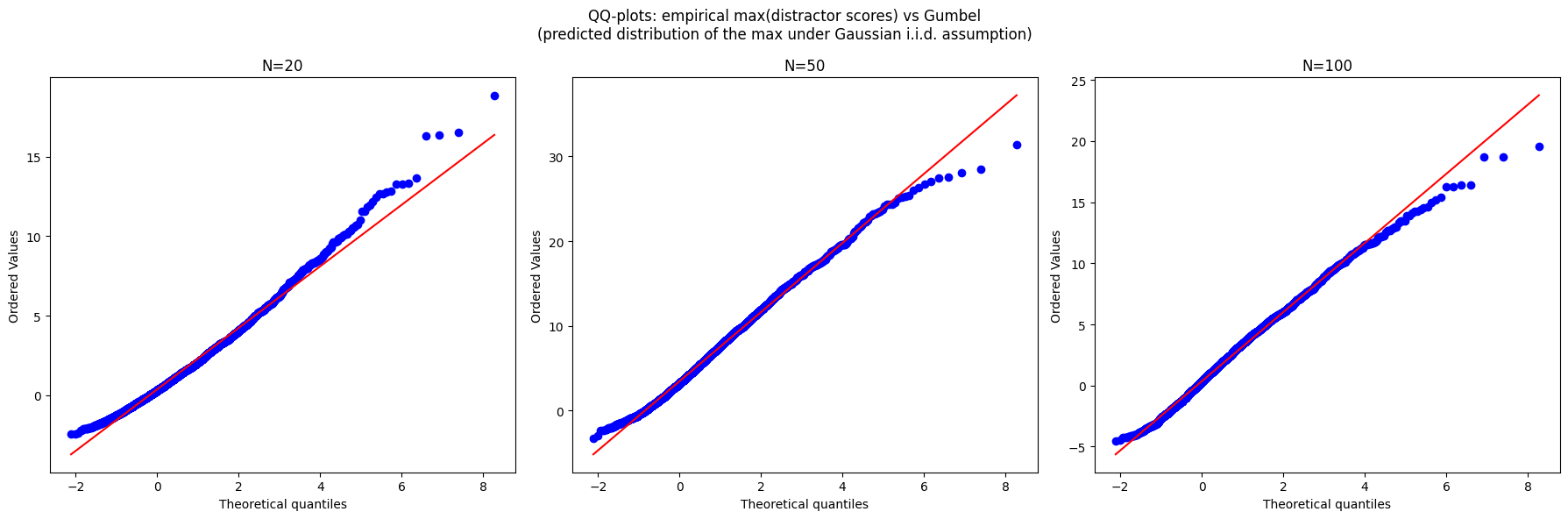}
    \caption{Q--Q plots of the empirical maximum standardized distractor
    scores against the chosen asymptotic Gumbel approximation, by $N$.}
    \label{fig:extreme}
\end{figure}

Figure~\ref{fig:extreme} compares the empirical maximum standardized
distractor scores with an asymptotic Gumbel approximation. Under an
exact i.i.d.\ standard-Gaussian model, the finite-$N$ maximum has CDF
$F_N(x)=\Phi(x)^N$; a Gumbel distribution arises only asymptotically
after appropriate centering and scaling. Rejection of the selected
Gumbel approximation at $N\in\{20,50,100\}$ therefore does not by
itself constitute a direct finite-sample test of the Gaussian i.i.d.\
hypothesis.

\begin{figure}[h]
    \centering
    \includegraphics[width=0.7\columnwidth]{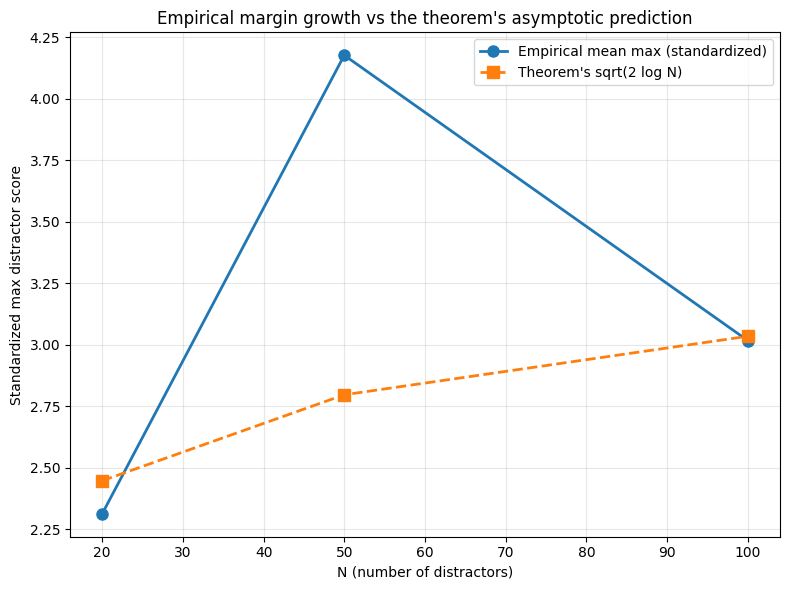}
    \caption{Empirical mean of the standardized maximum distractor score vs.\ the theorem's own asymptotic prediction, $\sqrt{2\log N}$.}
    \label{fig:margin}
\end{figure}

Despite this, the theorem's own margin-growth prediction (Figure~\ref{fig:margin}) tracks the data reasonably closely at two of three points---2.31 vs.\ 2.45 at $N=20$ and 3.02 vs.\ 3.04 at $N=100$---showing the abstraction captures the right order of magnitude even where its distributional assumptions do not hold exactly.

A natural question is whether the findings above reflect a general property of transformer attention or an artifact specific to Gemma-3's architecture, which interleaves sliding-window (local) attention layers with global-attention layers. We repeated the identical experiment on Llama-3-8B-Instruct \citep{grattafiori2024llama3}, which uses full causal (global) attention in every layer, with no sliding-window mechanism. Table~\ref{tab:model_arch} summarizes the two architectures.

\begin{table*}[h]
\centering
\caption{Architectural comparison of the two models tested.}
\label{tab:model_arch}
\small
\begin{tabular}{@{}lcc@{}}
\toprule
 & Gemma-3-4B-IT & Llama-3-8B-Instruct \\
\midrule
Parameters & $\sim$4B & $\sim$8B \\
Layers & 34 & 32 \\
Attention heads & 8 & 32 \\
Attention pattern & Mixed: sliding-window (local) + global layers & Full causal (global) attention, every layer \\
\bottomrule
\end{tabular}
\end{table*}

\begin{figure}[h]
    \centering
    \includegraphics[width=\columnwidth]{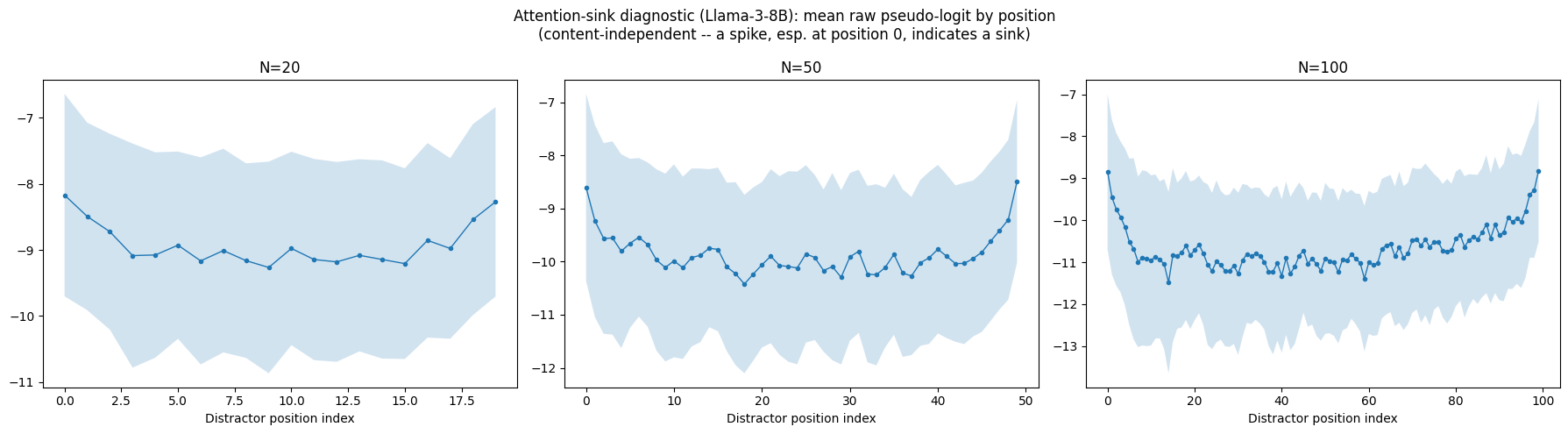}
    \caption{Attention-sink diagnostic on Llama-3-8B-Instruct (see Figure~\ref{fig:sink} for Gemma-3).}
    \label{fig:sink_llama}
\end{figure}

Figure~\ref{fig:sink_llama} shows a smooth positional profile at all
three $N$, with no sharp discontinuity at $N=100$. The absence of the
step in Llama-3 is consistent with an architecture-specific explanation
for the Gemma-3 pattern. However, this two-model comparison does not
isolate sliding-window/global layer mixing as the cause, because the
models also differ in size, training, tokenization, and other
architectural details.

\begin{table*}[h]
\centering
\caption{Attention distribution tests on Llama-3-8B-Instruct, by $N$ (32 layers $\times$ 32 heads, 10 trials per $N$).}
\label{tab:attention_tests_llama}
\small
\begin{tabular}{@{}lccc@{}}
\toprule
 & $N=20$ & $N=50$ & $N=100$ \\
\midrule
\multicolumn{4}{@{}l}{\textit{Gaussianity (marginal distribution)}} \\
Excess kurtosis (95\% CI)      & 2.32 [2.01, 2.71] & 0.98 [0.93, 1.03] & 0.89 [0.86, 0.92] \\
\addlinespace
\multicolumn{4}{@{}l}{\textit{Independence (lags 1--10)}} \\
Ljung--Box $Q$                 & 3.92  & 25.29 & 104.63 \\
Ljung--Box $p$                 & 0.951 & 0.0048 & $<10^{-6}$ \\
\addlinespace
\multicolumn{4}{@{}l}{\textit{Extreme value (max statistic vs.\ Gumbel)}} \\
KS statistic                   & 0.195 & 0.228 & 0.273 \\
Empirical mean max (std.)      & 2.29 & 2.63 & 2.96 \\
Leading-order scale $\sqrt{2\log N}$    & 2.45 & 2.80 & 3.04 \\
\bottomrule
\end{tabular}
\end{table*}

The cross-architecture comparison reveals model-specific departures from the Gaussian i.i.d.\ abstraction (Tables~\ref{tab:attention_tests} and~\ref{tab:attention_tests_llama}). Gemma-3's excess kurtosis increases with $N$ ($2.70\to9.12$), whereas Llama-3's decreases ($2.32\to0.89$). At $N=100$, the Ljung--Box test rejects independence for both models. Because the analyses pool different numbers of layer-head pairs, however, the raw $Q$ statistics are not directly comparable as measures of violation severity. For Gemma-3, the reported KS $p$-values reject an exact Gumbel fit. Table~\ref{tab:attention_tests_llama} reports KS statistics for Llama-3 but not the corresponding $p$-values, so we do not make a formal rejection claim for Llama-3 without those values.

Despite these departures, the empirical mean maximum for Llama-3 is $93\%$, $94\%$, and $97\%$ of the $\sqrt{2\log N}$ prediction at $N=20$, $50$, and $100$, respectively (Figure~\ref{fig:margin_llama}). Thus, the asymptotic scale remains a useful order-of-magnitude reference, even though the exact distributional assumptions are not satisfied.

\begin{figure}[h]
    \centering
    \includegraphics[width=0.7\columnwidth]{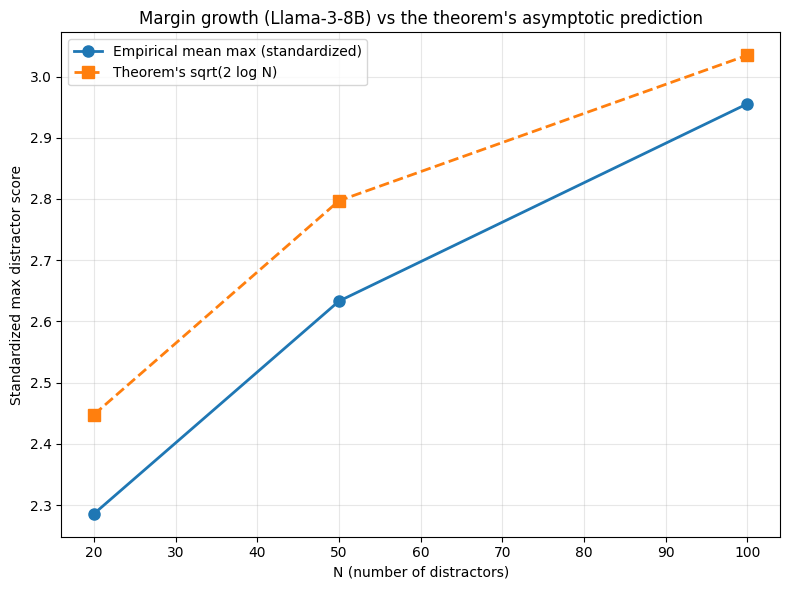}
    \caption{Margin growth on Llama-3-8B-Instruct vs.\ the theorem's asymptotic prediction (cf.\ Figure~\ref{fig:margin} for Gemma).}
    \label{fig:margin_llama}
\end{figure}

Does this internal violation translate into behavioral failure? For each trial we additionally generated the model's real answer and checked whether it contained the true code. Retrieval accuracy was $100\%$ at every $N$ (30/30 trials total; Figure~\ref{fig:retrieval}), while excess kurtosis rose monotonically over the same range. For this task---simple single-fact retrieval against generic filler---the internal assumption violation is not sufficient on its own to produce behavioral failure.

\begin{figure}[h]
    \centering
    \includegraphics[width=\columnwidth]{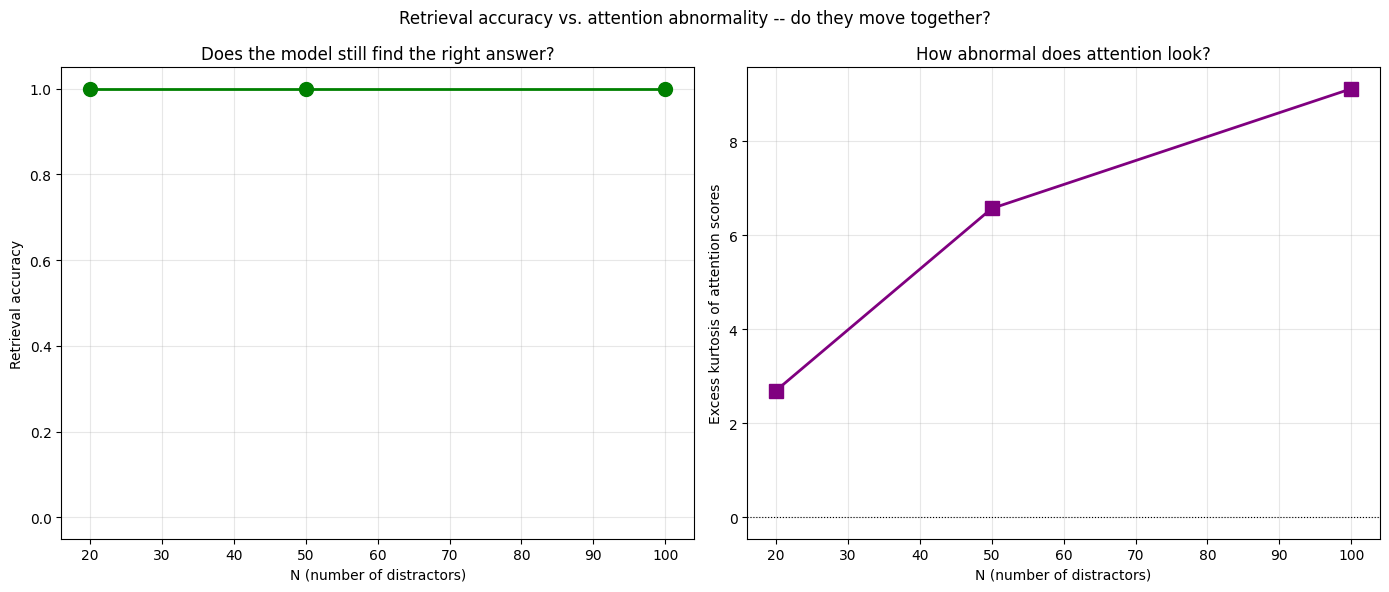}
    \caption{Retrieval accuracy (left) vs.\ excess kurtosis of the attention distribution (right), by $N$.}
    \label{fig:retrieval}
\end{figure}

If neither the assumption violation nor $N$ alone predicts failure, does raw context length matter at all? Extending purely behavioral testing (no attention extraction, hence tractable at much larger $N$) with the same generic-filler distractors to $N=100, 200, 400$ (up to $\sim$5,400 tokens), accuracy remained $100\%$ throughout. Increasing the number of generic-filler distractors from $100$ to $400$ did not reduce observed accuracy for this model on this task. This result applies only to the tested context range and distractor construction.

We therefore varied \emph{what} the distractors said rather than how many there were, holding $N \in \{20, 50, 100\}$ fixed. Low-similarity distractors (generic, topically unrelated filler) and medium-similarity distractors (the same numeric-code format, distinguished by a different label such as ``backup code'') both left accuracy at $100\%$ at every $N$ tested. High-similarity distractors---identical phrasing to the evidence sentence, differing only in the number---collapsed accuracy to near zero at every $N$ tested instead (0.00, 0.07, and 0.00 at $N=20$, $50$, and $100$ respectively). This approximately $N$-invariant pattern over the three tested values supports the paper's distinction (Definition~\ref{def:aliasing}) between raw distractor count $N$ and effective distractor count $N_{\mathrm{eff}}$: over the tested range, confusability is more predictive of failure than raw distractor count. Because this three-condition low/medium/high comparison cannot determine how much confusability is required, we then fixed $N=50$ total distractors and swept $k$, the number of high-similarity decoys mixed into an otherwise generic-filler haystack, from $0$ to $50$ (Table~\ref{tab:dose}, Figure~\ref{fig:dose}).

\begin{table*}[h]
\centering
\caption{Retrieval accuracy versus the number $k$ of confusable
decoys, with $N_{\mathrm{dist}}=50$ total distractors fixed.}
\label{tab:dose}
\small
\begin{tabular}{@{}cccccccccc@{}}
\toprule
$k$ & 0 & 1 & 2 & 3 & 5 & 10 & 20 & 35 & 50 \\
\midrule
Accuracy & 1.00 & 0.67 & 0.33 & 0.40 & 0.07 & 0.13 & 0.07 & 0.07 & 0.00 \\
\bottomrule
\end{tabular}
\end{table*}

Each point is based on $15$ trials, so one additional correct response changes the reported accuracy by $1/15\approx0.067$. Accordingly, the adjacent changes from $k=2$ ($0.33$) to $k=3$ ($0.40$), and from $k=5$ ($0.07$) to $k=10$ ($0.13$), each correspond to only one additional correct response. These local reversals should therefore not be interpreted as evidence of a non-monotonic underlying effect. The robust feature is the overall steep initial decline followed by a low plateau.

Accuracy drops sharply within the first two confusable decoys ($k=0$: $1.00$, $k=1$: $0.67$, $k=2$: $0.33$). A single confusable decoy among fifty distractors reduces observed accuracy by approximately one third, and two reduce it by approximately two thirds. From $k=5$ through $k=35$, observed accuracy remains in the $7\%$--$13\%$ range across a sevenfold increase in $k$. At $k=50$, when every distractor is confusable, observed accuracy is $0/15$. Thus, most of the observed degradation occurs after introducing a small number of confusable decoys, while additional decoys produce comparatively little further change.

\begin{figure}[h]
    \centering
    \includegraphics[width=0.8\columnwidth]{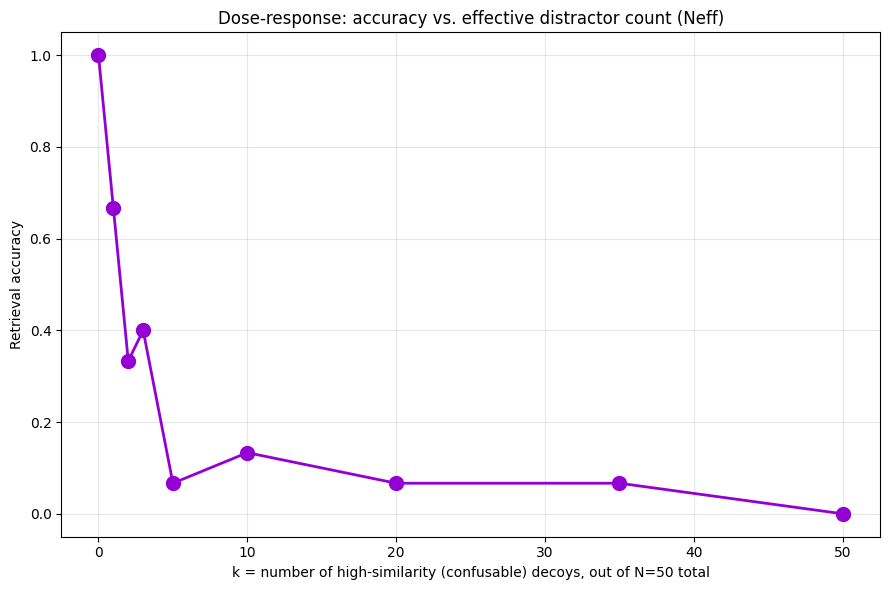}
    \caption{Retrieval accuracy versus the number $k$ of confusable decoys, with $N_{\mathrm{dist}}=50$ total distractors fixed.}    \label{fig:dose}
\end{figure}

In qualitative inspection, some failures under high similarity or larger $k$ included hedging behaviors that were not observed at $k=0$, such as listing multiple candidate codes or generating an unrequested confidence statement or rationale. Because these behaviors were not coded systematically, we treat them as anecdotal indications of uncertainty rather than as a measured result.

Taken together, these results provide a coherent but task-specific picture. The Gaussian i.i.d.\ assumption is not exact, and the form of its violation differs across the two architectures. Nevertheless, the theorem's margin-growth prediction remains close in order of magnitude to the observed maxima. In this retrieval task, behavioral failure tracks the number of textually confusable decoys more closely than raw context length over the tested range. Only a small number of such decoys is needed to produce most of the observed degradation. These findings support interpreting $N$ in the bound as an effective distractor count $N_{\rm eff}$, while not establishing that raw context length is irrelevant for other tasks, models, or longer-context regimes.

\section{Lost-in-the-middle reanalysis}
\label{app:lostmiddle}

Definition~\ref{def:curve} defines the poisoning curve $P(N) = A(N_0) - A(N)$: the drop in accuracy as context length $N$ grows. This has so far been supported partly by citing prior work \citep{liu2024lost}; here we re-run a version of their setup directly on our own model to report a concrete, reproducible number.

We use the publicly released data from \citet{liu2024lost}, which pair real questions from the Natural Questions dataset \citep{kwiatkowski2019natural} with Wikipedia passages retrieved using Contriever-MSMARCO \citep{izacard2021contriever}. Each example has one gold passage containing the true answer and several distractor passages. We sampled 40 fixed questions, reused across every condition so results are paired rather than confounded by question difficulty, and evaluated them on Gemma-3-4B-IT \citep{gemmateam2025gemma3} (Hugging Face \texttt{transformers}, greedy decoding). A model's answer is scored correct if it contains the known gold answer string.

Two conditions were tested. In the position condition,
$N_{\mathrm{docs}}=20$ documents are fixed and the gold passage is
moved to zero-indexed position $0$, $4$, $9$, $14$, or $19$. In
the scaling condition, the gold passage is held at zero-indexed
position $0$, which had the highest observed accuracy in the first
condition, while $N_{\mathrm{docs}}\in\{10,20,30\}$ is varied. The
second condition therefore measures the effect of document count
while holding the gold passage's position fixed.

\begin{figure}[h]
    \centering
    \includegraphics[width=0.65\columnwidth]{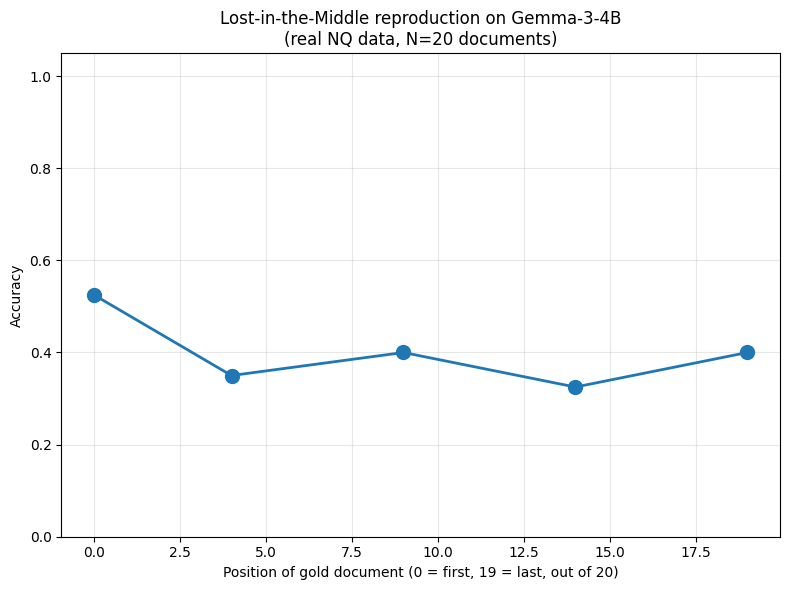}
    \caption{Accuracy vs.\ position of the gold document, $N=20$, 40 real NQ questions.}
    \label{fig:position}
\end{figure}

Accuracy is highest when the gold passage is first ($52.5\%$) and
ranges from $32.5\%$ to $40.0\%$ at the other tested positions,
a decrease of $12.5$--$20.0$ percentage points
(Figure~\ref{fig:position}). This provides descriptive evidence of
a position effect in this sample and is consistent with
\citet{liu2024lost}.

\begin{table}[h]
\centering
\caption{Accuracy and poisoning curve $P(N) = A(10) - A(N)$, gold passage fixed at position 0.}
\label{tab:poisoning}
\small
\begin{tabular}{@{}lccc@{}}
\toprule
$N$ & 10 & 20 & 30 \\
\midrule
$A(N)$ & 0.550 & 0.525 & 0.550 \\
$P(N)$ & 0.000 & 0.025 & 0.000 \\
\bottomrule
\end{tabular}
\end{table}

\begin{figure}[h]
    \centering
    \includegraphics[width=\columnwidth]{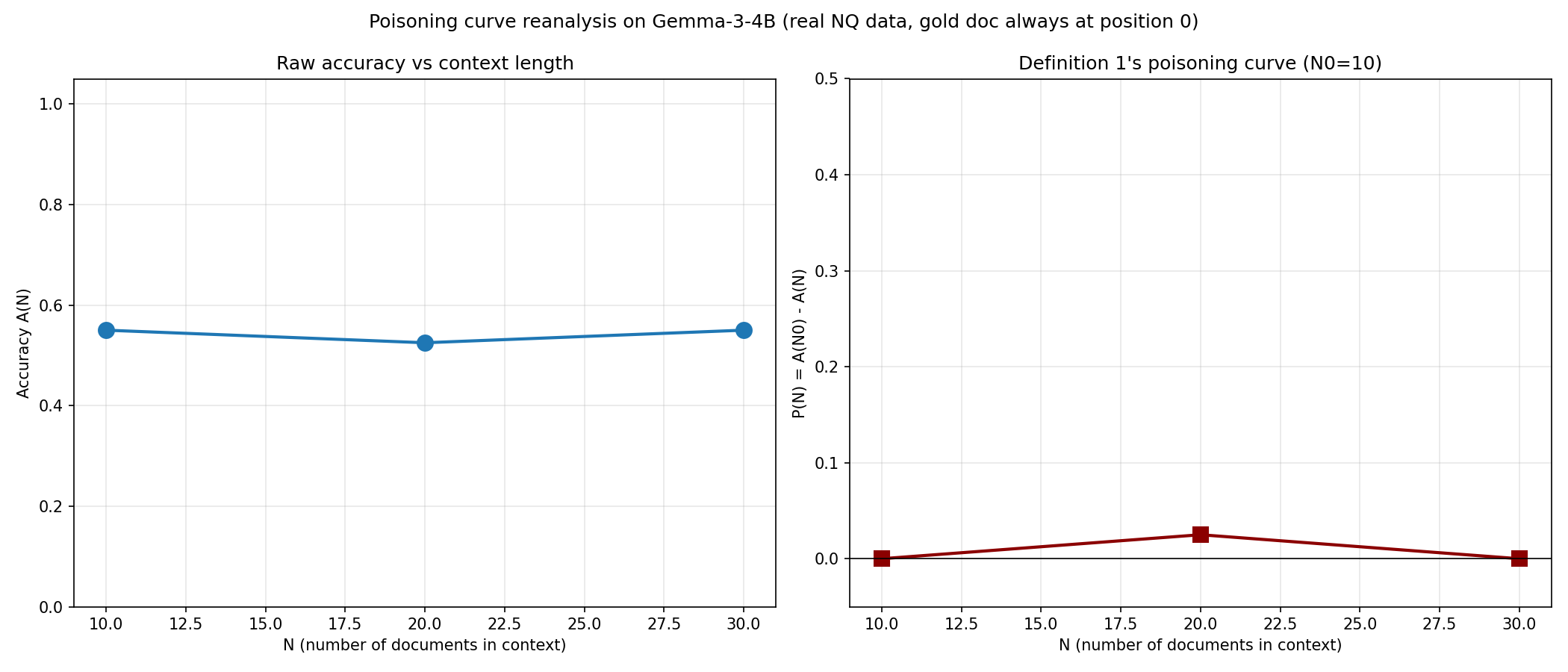}
    \caption{Accuracy and $P(N)$ vs.\ context length, gold passage fixed at position 0.}
    \label{fig:poisoning}
\end{figure}

With the gold passage fixed at position $0$, which had the highest observed accuracy, $P(N)$ remains close to zero across $N\in\{10,20,30\}$ (Table~\ref{tab:poisoning}, Figure~\ref{fig:poisoning}). We therefore observe no monotonic decrease in accuracy with document count over this range. Together with the position experiment, these results indicate that position has a larger observed effect than raw document count for this model and within the tested range. They do not establish that position remains the dominant factor at larger context sizes or for other models. The result is compatible with Theorem~\ref{thm:poisoning}'s emphasis on effective distractor exposure and Definition~\ref{def:aliasing}, but this experiment does not measure $N_{\rm eff}$ directly.

Two limitations are worth noting. First, the study uses only $40$ questions, so the estimates have substantial uncertainty. Second, the scaling condition tests only up to $N=30$ documents. Whether $P(N)$ remains near zero at larger values of $N$, or whether position and length effects interact at greater scale, remains unresolved.


\section{Empirical Test of the Gate-Mitigation Prediction}
\label{sec:gate-experiment}

Proposition~\ref{prop:gate} predicts that replacing a raw candidate set
of $N$ context chunks with a bounded retrieved set of $K\!\ll\!N$
chunks trades context-length interference against a retrieval-error
penalty $\hat{\eta}=1-\text{Recall@}K$. We test this prediction
directly, following the reporting protocol in
\nameref{sec:evaluation-protocol}: we measure answer accuracy and
evidence recall separately, and we isolate the model's use of retrieved
evidence from the quality of retrieval itself. We emphasize the scope
of the experiment: it tests the qualitative tradeoff described by the
proposition, but it does not validate the theorem's Gaussian-logit
assumptions or its $\sqrt{\log N}$ asymptotic rate.

\subsection{Design}
\label{sec:gate-design}

\paragraph{Benchmark.}

We use RULER's \texttt{qa\_2} task \citep{hsieh2024ruler}, which is built on HotpotQA
\citep{yang2018hotpotqa}. Each instance is a multi-hop question whose two gold supporting paragraphs
are embedded among distractor paragraphs presented as numbered
\texttt{Document $n$:} blocks. We construct instances directly from
the HotpotQA distractor development set to preserve HotpotQA's gold
supporting-fact titles, which allows us to measure evidence recall
rather than relying on an answer-string proxy.

Starting from each question's native two gold and eight distractor
paragraphs, we add distractor paragraphs drawn from a global pool until
a target token length is reached. We evaluate seven lengths,
$\{4,8,16,32\}$K and $\{128,256,512\}$K tokens, using the same
$100$ questions at every length so that comparisons are paired across
both conditions and lengths. Lengths are measured with the
\texttt{o200k\_base} tokenizer. The median raw candidate count $N$
grows from $29$ chunks at $4$K to $3{,}711$ chunks at $512$K
(Table~\ref{tab:main}).

\paragraph{Conditions.}
Each instance is answered under two conditions that differ only in the documents supplied. The \textbf{raw} condition passes all $N$
documents in source order. The \textbf{gated} condition scores each document against the question using BM25 \citep{robertson2009probabilistic}, retains the top-$K$
documents, and restores them to their original source order before
answering. Restoring source order ensures that the intervention changes
context length without also changing document ordering.

Both conditions use the same prompt template, system instruction, and
decoding configuration: temperature $0$, top-$p$ $1$, a maximum of
$64$ generated tokens, and one generation call. We deliberately omit
a quote-the-evidence-first step from the gated condition because such a
step would confound retrieval gating with an additional inference
procedure.

\paragraph{Gate and $K$ selection.}
Each \texttt{Document $n$:} block is treated as one retrieval chunk.
Queries and documents use identical lowercasing and tokenization for
BM25 scoring. Because the benchmark is multi-hop, the gate succeeds
only when it retrieves all annotated supporting documents:
\begin{equation}
\mathrm{Hit}_i(K)
=
\mathrm{Hit}_i(K)
=
\1\!\left\{G_i\subseteq R_i(K)\right\}.
\end{equation}
The estimated retrieval penalty is
\begin{equation}
\hat{\eta}
=
1-
\frac{1}{M}
\sum_{i=1}^{M}
\mathrm{Hit}_i(K).
\end{equation}

On a disjoint $100$-question development split, we sweep
$K\in\{2,4,8,16\}$ and select the smallest value achieving
all-support Recall@$K\geq0.90$ at every length through $32$K. This
procedure selects $K=16$. Retrieval and selection of $K$ are
performed offline and require no test-set model calls.

\paragraph{Model.}
All answers are generated by GPT-4.1
(\texttt{gpt-4.1\_2025-04-14}) at temperature $0$. We verified that
the serving endpoint accepts inputs of at least $520$K tokens, and no
raw prompt was truncated. The study consists of $1{,}400$
generations, forming $700$ paired raw--gated comparisons
($100\times7$ pairs), with zero endpoint errors.

\subsection{Metrics}
\label{sec:gate-metrics}

We report normalized exact match (EM) and token F1 using the official
HotpotQA scorer. The primary effect is the paired difference
\begin{equation}
\Delta\mathrm{EM}(N,K)
=
A_{\mathrm{gate}}-A_{\mathrm{raw}},
\end{equation}
with $95\%$ confidence intervals computed from $10{,}000$ paired
bootstrap samples and paired binary outcomes evaluated using an exact
McNemar test.

To separate evidence use from retrieval quality, we also report
conditional accuracy on the retrieval-hit subset and its difference
\begin{equation}
\Delta_{\mathrm{use}}
=
A_{\mathrm{gate}\mid\mathrm{hit}}
-
A_{\mathrm{raw}\mid\mathrm{hit}}.
\end{equation}
On this subset, the gate is known to have retained all annotated
supporting evidence. A positive $\Delta_{\mathrm{use}}$ is therefore
consistent with a benefit from reducing competing context rather than
merely from successful retrieval.

\subsection{Results}
\label{sec:gate-results}

\begin{table*}[t]
\centering
\scriptsize
\setlength{\tabcolsep}{2.8pt}
\caption{GPT-4.1 on RULER \texttt{qa\_2}/HotpotQA: raw long context versus a BM25 top-$K$ retrieval gate with $K=16$, frozen on a disjoint development split. Each length uses the same $100$ questions at temperature $0$. The retrieval penalty is $\hat{\eta}=1-\text{Recall@}K$. $\Delta\mathrm{EM}$ is the paired gated-minus-raw difference with a $10{,}000$-sample paired-bootstrap $95\%$ confidence interval. $\Delta_{\mathrm{use}}$ is the gated-minus-raw difference restricted to instances for which all annotated supporting evidence was retrieved. McNemar $p$ is the exact two-sided test.}
\label{tab:main}
\resizebox{\textwidth}{!}{%
\begin{tabular}{lrr*{11}{c}}
\toprule
Ctx & \shortstack{Median\\raw $N$} & $K$ & \shortstack{Recall\\@$K$} & $\hat{\eta}$ & \shortstack{Raw\\EM} & \shortstack{Gate\\EM} & $\Delta\mathrm{EM}$ [95\% CI] & \shortstack{Raw\\F1} & \shortstack{Gate\\F1} & \shortstack{Gate\\EM$\mid h$} & \shortstack{Raw\\EM$\mid h$} & $\Delta_{\mathrm{use}}$ & \shortstack{McNemar\\$p$} \\
\midrule
4K   & 29   & 16 & 0.94 & 0.06 & 0.660 & 0.640 & $-0.020$ [$-0.070$, $+0.030$] & 0.768 & 0.751 & 0.681 & 0.660 & 0.021  & 0.727 \\
8K   & 58   & 16 & 0.94 & 0.06 & 0.650 & 0.620 & $-0.030$ [$-0.080$, $+0.020$] & 0.768 & 0.738 & 0.660 & 0.660 & 0.000  & 0.453 \\
16K  & 115  & 16 & 0.94 & 0.06 & 0.680 & 0.630 & $-0.050$ [$-0.110$, $+0.010$] & 0.792 & 0.723 & 0.670 & 0.691 & $-0.021$ & 0.180 \\
32K  & 232  & 16 & 0.92 & 0.08 & 0.650 & 0.650 & $+0.000$ [$-0.050$, $+0.050$] & 0.782 & 0.757 & 0.696 & 0.674 & 0.022  & 1.000 \\
128K & 930  & 16 & 0.86 & 0.14 & 0.590 & 0.640 & $+0.050$ [$-0.010$, $+0.110$] & 0.724 & 0.746 & 0.709 & 0.640 & 0.070  & 0.227 \\
256K & 1857 & 16 & 0.82 & 0.18 & 0.590 & 0.580 & $-0.010$ [$-0.090$, $+0.070$] & 0.715 & 0.694 & 0.683 & 0.622 & 0.061  & 1.000 \\
512K & 3711 & 16 & 0.79 & 0.21 & 0.570 & 0.590 & $+0.020$ [$-0.060$, $+0.100$] & 0.704 & 0.705 & 0.709 & 0.595 & 0.114  & 0.815 \\
\bottomrule
\end{tabular}%
}
\end{table*}

\begin{figure*}[t]
\centering
\includegraphics[width=\textwidth]{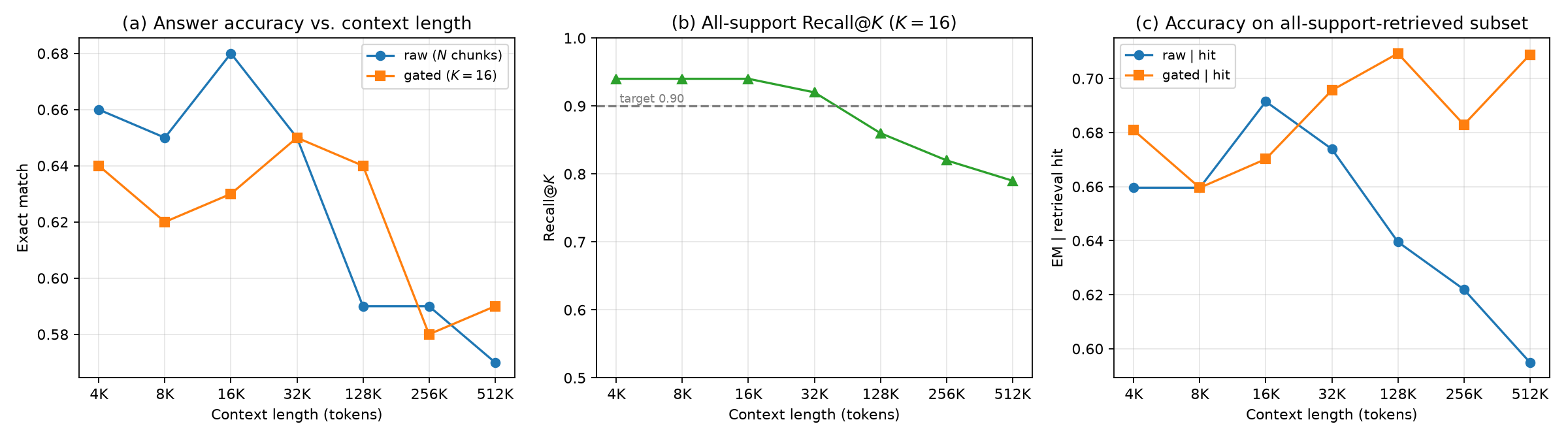}
\caption{GPT-4.1 with $K=16$. \textbf{(a)} Raw accuracy is approximately stable through $32$K and then decreases with context length, while gated accuracy varies less. \textbf{(b)} A fixed $K=16$ meets the $0.90$ all-support recall target through $32$K, after which recall decreases as the candidate set grows. \textbf{(c)} On instances for which the gate retrieves all annotated supporting evidence, raw accuracy decreases with context length while gated accuracy remains comparatively stable, providing evidence of improved evidence use after competing context is removed.}
\label{fig:main}
\end{figure*}

Table~\ref{tab:main} and Figure~\ref{fig:main} support four findings.

\paragraph{(1) No measurable poisoning through $32$K.}
Up to $32$K, raw EM remains approximately stable
($0.66\!\to\!0.65$), and the gate provides no statistically detectable
net benefit. The observed $\Delta\mathrm{EM}$ values remain within
$[-0.05,+0.05]$, all reported confidence intervals include zero, and
McNemar $p\geq0.18$.

\paragraph{(2) Degradation appears beyond $32$K.}
From $32$K to $512$K, raw EM decreases from $0.65$ to $0.57$ and
raw F1 decreases from $0.78$ to $0.70$. No comparable downward trend
is observed among the tested lengths from $4$K through $32$K. The
gated condition exhibits a smaller overall decrease across the longest
tested contexts.

\paragraph{(3) The gate improves evidence use on retrieval-hit instances.}
Figure~\ref{fig:main}(c) restricts evaluation to questions for which
the gate retrieves all annotated supporting evidence. On this subset,
raw accuracy decreases from $0.69$ at $16$K to $0.595$ at $512$K,
whereas gated accuracy remains near $0.71$. The corresponding
$\Delta_{\mathrm{use}}$ values are $+0.070$, $+0.061$, and $+0.114$
at $128$K, $256$K, and $512$K, respectively. Because both conditions
contain all annotated gold evidence on this subset, the pattern is
consistent with a benefit from removing competing distractor context,
rather than merely from successful retrieval.

\paragraph{(4) A fixed $K$ incurs a growing retrieval penalty.}
With $K$ fixed at $16$, all-support recall decreases from $0.94$ to
$0.79$ as the raw candidate count grows to approximately $3{,}700$
chunks. Correspondingly, $\hat{\eta}$ increases from $0.06$ to $0.21$.
The evidence-use benefit is therefore offset by the $14\%$--$21\%$ of
instances for which the bounded gate omits required evidence, and the
net $\Delta\mathrm{EM}$ remains near zero at the longest contexts.
This pattern matches the qualitative tradeoff in
Proposition~\ref{prop:gate}: reducing the candidate set from $N$ to
$K$ can improve evidence use, but that benefit is offset when the
retrieval penalty $\hat{\eta}$ grows with the candidate set.

\subsection{Scaling $K$ to Preserve Evidence Recall}
\label{sec:gate-largerk}

To separate the evidence-use benefit from the retrieval penalty, we
re-run only the gated condition at $128$K, $256$K, and $512$K with a
length-specific value of $K$ selected on the development split to
restore all-support Recall@$K\geq0.90$. The resulting values are
$K=64$, $128$, and $192$, respectively. The raw predictions are
unchanged and reused from the primary experiment. Even at $K=192$,
the gated prompt remains approximately $25$K tokens, compared with
approximately $504$K tokens in the raw condition.

\begin{table*}[t]
\centering
\small
\setlength{\tabcolsep}{4pt}
\caption{Scaling $K$ to restore recall at extreme lengths. Only the gated arm is re-run; raw predictions are cached and unchanged. Each length-specific value of $K$ is selected on the development split to achieve Recall@$K\geq0.90$.}
\label{tab:largerk}
\begin{tabular}{lrcccccc}
\toprule
Ctx & $K$ & Recall@$K$ & Raw EM & Gate EM & $\Delta\mathrm{EM}$ [95\% CI] & $\Delta_{\mathrm{use}}$ & McNemar $p$ \\
\midrule
128K & 16  & 0.86 & 0.590 & 0.640 & $+0.050$ [$-0.010$, $+0.110$] & 0.070 & 0.227 \\
     & 64  & 0.94 & 0.590 & 0.620 & $+0.030$ [$-0.020$, $+0.080$] & 0.043 & 0.453 \\
\addlinespace
256K & 16  & 0.82 & 0.590 & 0.580 & $-0.010$ [$-0.090$, $+0.070$] & 0.061 & 1.000 \\
     & 128 & 0.94 & 0.590 & 0.640 & $+0.050$ [$-0.010$, $+0.120$] & 0.064 & 0.227 \\
\addlinespace
512K & 16  & 0.79 & 0.570 & 0.590 & $+0.020$ [$-0.060$, $+0.100$] & 0.114 & 0.815 \\
     & 192 & 0.93 & 0.570 & 0.640 & $+0.070$ [$+0.000$, $+0.140$] & 0.097 & 0.092 \\
\bottomrule
\end{tabular}
\end{table*}

\begin{figure*}[t]
\centering
\includegraphics[width=\textwidth]{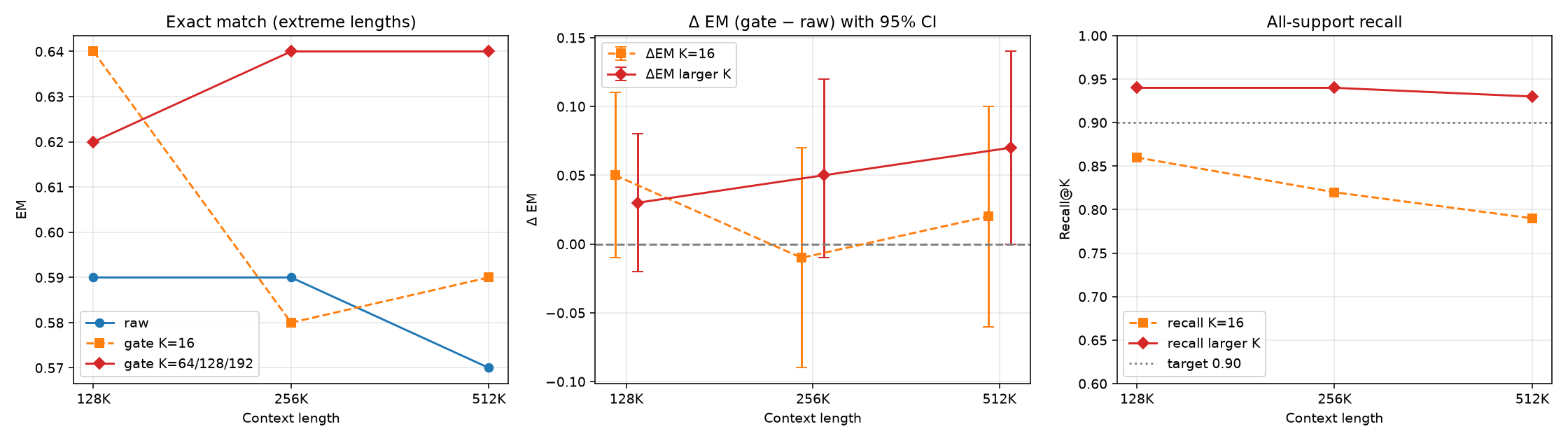}
\caption{Comparison between fixed $K=16$ and a retrieval budget scaled to maintain all-support Recall@$K$ of at least $0.90$ at extreme context lengths. \textbf{Left:} gated accuracy remains comparatively stable under the scaled retrieval budget while raw accuracy decreases. \textbf{Middle:} the paired gated-minus-raw difference becomes more positive after the retrieval penalty is reduced. \textbf{Right:} the scaled values of $K$ maintain recall above the target where fixed $K=16$ does not.}
\label{fig:largerk}
\end{figure*}

As shown in Table~\ref{tab:largerk} and
Figure~\ref{fig:largerk}, restoring recall makes the net gate advantage
more positive at the longest tested contexts. At $512$K, increasing
$K$ from $16$ to $192$ raises gated EM from $0.59$ to $0.64$, while
raw EM is $0.57$. The corresponding $\Delta\mathrm{EM}$ increases
from $+0.02$ to $+0.070$, with a $95\%$ confidence interval of
$[+0.000,+0.140]$ and McNemar $p=0.092$. At $256$K,
$\Delta\mathrm{EM}$ increases from $-0.01$ to $+0.050$.

At $128$K, fixed $K=16$ already retains most annotated support
(recall $0.86$), and increasing $K$ slightly reduces
$\Delta\mathrm{EM}$ from $+0.05$ to $+0.03$, consistent with
reintroducing additional distractors. In the logistic model
\begin{equation}
\mathrm{correct}
\sim
\mathrm{condition}
\times
\log_2 N,
\end{equation}
the estimated gate-by-length interaction increases from $+0.035$ with
fixed $K=16$ to $+0.084$ under the scaled retrieval budget. These
point estimates suggest that the retrieval budget may need to grow
with the candidate set to maintain evidence recall, rather than being
held fixed independently of $N$.

\subsection{Discussion and Limitations}
\label{sec:gate-discussion}

The experiment supports the gate-mitigation prediction within the
tested regime. For GPT-4.1, no measurable degradation is observed
through $32$K. At longer contexts, raw accuracy decreases, while the
retrieval-hit analysis indicates that the gate can improve the model's
use of evidence that remains available. The net accuracy benefit
becomes more positive when $K$ is increased to keep the retrieval
penalty $\hat{\eta}$ small. Together, these results are consistent
with the $N\!\to\!K$ versus $\hat{\eta}$ tradeoff described by
Proposition~\ref{prop:gate}.

Several limitations bound these claims. First, with $M=100$ questions,
the strongest net effect at $512$K remains borderline: the confidence
interval has a lower bound of zero and McNemar $p\approx0.09$. A
larger evaluation set would improve precision; under standard
square-root scaling, doubling the sample size would reduce standard
errors by approximately a factor of $1/\sqrt{2}$, not by one half.

Second, we study a single hosted model. Other models may enter the
degradation regime at different context lengths, so replication across
model families and deployment settings is needed to assess the
generality of finding~(2).

Third, contexts are constructed and measured using a tokenizer proxy
rather than the proprietary serving tokenizer. The reported lengths
are therefore internally consistent but may not exactly equal
server-side token counts. In addition, temperature-$0$ hosted decoding
is not necessarily perfectly deterministic.

Finally, as stated at the outset, these results test the proposition's
qualitative tradeoff rather than the theorem's distributional
assumptions or asymptotic rate.

%% file: aaai2027.bib
@inproceedings{an2024leval,
  title         = {{L}-Eval: Instituting Standardized Evaluation for Long Context Language Models},
  author        = {An, Chenxin and Gong, Shansan and Zhong, Ming and Zhao, Xingjian and Li, Mukai and Zhang, Jun and Kong, Lingpeng and Qiu, Xipeng},
  booktitle     = {Proceedings of the 62nd Annual Meeting of the Association for Computational Linguistics},
  year          = {2024}
}

@inproceedings{bai2024longbench,
  title         = {LongBench: A Bilingual, Multitask Benchmark for Long Context Understanding},
  author        = {Bai, Yushi and Lv, Xin and Zhang, Jiajie and Lyu, Hongchang and Tang, Jiankai and Huang, Zhidian and Du, Zhengxiao and Liu, Xiao and Zeng, Aohan and Hou, Lei and Dong, Yuxiao and Tang, Jie and Li, Juanzi},
  booktitle     = {Proceedings of the 62nd Annual Meeting of the Association for Computational Linguistics},
  year          = {2024}
}

@misc{beltagy2020longformer,
  title         = {Longformer: The Long-Document Transformer},
  author        = {Beltagy, Iz and Peters, Matthew E. and Cohan, Arman},
  year          = {2020},
  eprint        = {2004.05150},
  archiveprefix = {arXiv},
  primaryclass  = {cs.CL}
}

@inproceedings{brown2020language,
  title         = {Language Models Are Few-Shot Learners},
  author        = {Brown, Tom B. and Mann, Benjamin and Ryder, Nick and Subbiah, Melanie and Kaplan, Jared and Dhariwal, Prafulla and Neelakantan, Arvind and Shyam, Pranav and Sastry, Girish and Askell, Amanda and others},
  booktitle     = {Advances in Neural Information Processing Systems},
  volume        = {33},
  pages         = {1877--1901},
  year          = {2020}
}

@misc{chen2023positionalinterpolation,
  title         = {Extending Context Window of Large Language Models via Positional Interpolation},
  author        = {Chen, Shouyuan and Wong, Sherman and Chen, Liangjian and Tian, Yuandong},
  year          = {2023},
  eprint        = {2306.15595},
  archiveprefix = {arXiv},
  primaryclass  = {cs.CL}
}

@inproceedings{choromanski2021rethinking,
  title         = {Rethinking Attention with Performers},
  author        = {Choromanski, Krzysztof and Likhosherstov, Valerii and Dohan, David and Song, Xingyou and Gane, Andreea and Sarlos, Tamas and Hawkins, Peter and Davis, Jared and Mohiuddin, Afroz and Kaiser, Lukasz and Belanger, David and Colwell, Lucy and Weller, Adrian},
  booktitle     = {International Conference on Learning Representations},
  year          = {2021}
}

@inproceedings{dai2019transformerxl,
  title         = {Transformer-{XL}: Attentive Language Models beyond a Fixed-Length Context},
  author        = {Dai, Zihang and Yang, Zhilin and Yang, Yiming and Carbonell, Jaime and Le, Quoc V. and Salakhutdinov, Ruslan},
  booktitle     = {Proceedings of the 57th Annual Meeting of the Association for Computational Linguistics},
  pages         = {2978--2988},
  year          = {2019}
}

@inproceedings{dao2022flashattention,
  title         = {FlashAttention: Fast and Memory-Efficient Exact Attention with {IO}-Awareness},
  author        = {Dao, Tri and Fu, Daniel Y. and Ermon, Stefano and Rudra, Atri and Re, Christopher},
  booktitle     = {Advances in Neural Information Processing Systems},
  volume        = {35},
  pages         = {16344--16359},
  year          = {2022}
}

@inproceedings{devlin2019bert,
  title         = {{BERT}: Pre-training of Deep Bidirectional Transformers for Language Understanding},
  author        = {Devlin, Jacob and Chang, Ming-Wei and Lee, Kenton and Toutanova, Kristina},
  booktitle     = {Proceedings of NAACL-HLT},
  pages         = {4171--4186},
  year          = {2019}
}

@inproceedings{du2025context,
  title         = {Context Length Alone Hurts {LLM} Performance Despite Perfect Retrieval},
  author        = {Du, Yufeng and Tian, Minyang and Ronanki, Srikanth and Rongali, Subendhu and Bodapati, Sravan and Galstyan, Aram and Wells, Azton and Schwartz, Roy and Huerta, Eliu A. and Peng, Hao},
  booktitle     = {Findings of the Association for Computational Linguistics: EMNLP},
  pages         = {23281--23298},
  year          = {2025}
}

@article{gemmateam2025gemma3,
  title   = {{Gemma 3} Technical Report},
  author  = {{Gemma Team}},
  journal = {arXiv preprint arXiv:2503.19786},
  year    = {2025}
}

@inproceedings{guu2020realm,
  title         = {{REALM}: Retrieval-Augmented Language Model Pre-Training},
  author        = {Guu, Kelvin and Lee, Kenton and Tung, Zora and Pasupat, Panupong and Chang, Ming-Wei},
  booktitle     = {International Conference on Machine Learning},
  pages         = {3929--3938},
  year          = {2020}
}

@inproceedings{hsieh2024ruler,
  title         = {{RULER}: What's the Real Context Size of Your Long-Context Language Models?},
  author        = {Hsieh, Cheng-Ping and Sun, Simeng and Kriman, Samuel and Acharya, Shantanu and Rekesh, Dima and Jia, Fei and Ginsburg, Boris},
  booktitle     = {Conference on Language Modeling (COLM)},
  year          = {2024},
  eprint        = {2404.06654},
  archiveprefix = {arXiv},
  primaryclass  = {cs.CL}
}

@inproceedings{izacard2021leveraging,
  title         = {Leveraging Passage Retrieval with Generative Models for Open Domain Question Answering},
  author        = {Izacard, Gautier and Grave, Edouard},
  booktitle     = {Proceedings of the 16th Conference of the European Chapter of the Association for Computational Linguistics},
  pages         = {874--880},
  year          = {2021}
}

@inproceedings{jiang2023longllmlingua,
  title         = {Long{LLM}Lingua: Accelerating and Enhancing {LLMs} in Long Context Scenarios via Prompt Compression},
  author        = {Jiang, Huiqiang and Wu, Qianhui and Luo, Xufang and Li, Dongsheng and Lin, Chin-Yew and Yang, Yuqing and Qiu, Lili},
  booktitle     = {Proceedings of the 2023 Conference on Empirical Methods in Natural Language Processing},
  year          = {2023}
}

@inproceedings{katharopoulos2020transformers,
  title         = {Transformers are {RNNs}: Fast Autoregressive Transformers with Linear Attention},
  author        = {Katharopoulos, Angelos and Vyas, Apoorv and Pappas, Nikolaos and Fleuret, Francois},
  booktitle     = {International Conference on Machine Learning},
  pages         = {5156--5165},
  year          = {2020}
}

@article{kwiatkowski2019natural,
  title         = {Natural Questions: A Benchmark for Question Answering Research},
  author        = {Kwiatkowski, Tom and Palomaki, Jennimaria and Redfield, Olivia and Collins, Michael and Parikh, Ankur and Alberti, Chris and Epstein, Danielle and Polosukhin, Illia and Devlin, Jacob and Lee, Kenton and Toutanova, Kristina and Jones, Llion and Kelcey, Matthew and Chang, Ming-Wei and Dai, Andrew M. and Uszkoreit, Jakob and Le, Quoc and Petrov, Slav},
  journal       = {Transactions of the Association for Computational Linguistics},
  volume        = {7},
  pages         = {453--466},
  year          = {2019}
}

@inproceedings{lewis2020rag,
  title         = {Retrieval-Augmented Generation for Knowledge-Intensive {NLP} Tasks},
  author        = {Lewis, Patrick and Perez, Ethan and Piktus, Aleksandra and Petroni, Fabio and Karpukhin, Vladimir and Goyal, Naman and Kuttler, Heinrich and Lewis, Mike and Yih, Wen-tau and Rocktaschel, Tim and Riedel, Sebastian and Kiela, Douwe},
  booktitle     = {Advances in Neural Information Processing Systems},
  volume        = {33},
  pages         = {9459--9474},
  year          = {2020}
}

@article{liu2024lost,
  title         = {Lost in the Middle: How Language Models Use Long Contexts},
  author        = {Liu, Nelson F. and Lin, Kevin and Hewitt, John and Paranjape, Ashwin and Bevilacqua, Michele and Petroni, Fabio and Liang, Percy},
  journal       = {Transactions of the Association for Computational Linguistics},
  volume        = {12},
  pages         = {157--173},
  year          = {2024},
  doi           = {10.1162/tacl_a_00638}
}

@misc{peng2023yarn,
  title         = {{YaRN}: Efficient Context Window Extension of Large Language Models},
  author        = {Peng, Bowen and Quesnelle, Jeffrey and Fan, Honglu and Shippole, Enrico},
  year          = {2023},
  eprint        = {2309.00071},
  archiveprefix = {arXiv},
  primaryclass  = {cs.CL}
}

@inproceedings{press2022train,
  title         = {Train Short, Test Long: Attention with Linear Biases Enables Input Length Extrapolation},
  author        = {Press, Ofir and Smith, Noah A. and Lewis, Mike},
  booktitle     = {International Conference on Learning Representations},
  year          = {2022}
}

@inproceedings{rae2020compressive,
  title         = {Compressive Transformers for Long-Range Sequence Modelling},
  author        = {Rae, Jack W. and Potapenko, Anna and Jayakumar, Siddhant M. and Lillicrap, Timothy P.},
  booktitle     = {International Conference on Learning Representations},
  year          = {2020}
}

@article{raffel2020exploring,
  title         = {Exploring the Limits of Transfer Learning with a Unified Text-to-Text Transformer},
  author        = {Raffel, Colin and Shazeer, Noam and Roberts, Adam and Lee, Katherine and Narang, Sharan and Matena, Michael and Zhou, Yanqi and Li, Wei and Liu, Peter J.},
  journal       = {Journal of Machine Learning Research},
  volume        = {21},
  number        = {140},
  pages         = {1--67},
  year          = {2020}
}

@article{robertson2009probabilistic,
  title         = {The Probabilistic Relevance Framework: {BM25} and Beyond},
  author        = {Robertson, Stephen and Zaragoza, Hugo},
  journal       = {Foundations and Trends in Information Retrieval},
  volume        = {3},
  number        = {4},
  pages         = {333--389},
  year          = {2009}
}

@inproceedings{shaham2023zeroscrolls,
  title         = {{ZeroSCROLLS}: A Zero-Shot Benchmark for Long Text Understanding},
  author        = {Shaham, Uri and Ivgi, Maor and Efrat, Avia and Berant, Jonathan and Levy, Omer},
  booktitle     = {Findings of the Association for Computational Linguistics: EMNLP},
  year          = {2023},
  eprint        = {2305.14196},
  archiveprefix = {arXiv},
  primaryclass  = {cs.CL}
}

@inproceedings{shi2023large,
  title         = {Large Language Models Can Be Easily Distracted by Irrelevant Context},
  author        = {Shi, Freda and Chen, Xinyun and Misra, Kanishka and Scales, Nathan and Dohan, David and Chi, Ed H. and Scharli, Nathanael and Zhou, Denny},
  booktitle     = {International Conference on Machine Learning},
  pages         = {31210--31227},
  year          = {2023}
}

@article{su2024roformer,
  title         = {RoFormer: Enhanced Transformer with Rotary Position Embedding},
  author        = {Su, Jianlin and Lu, Yu and Pan, Shengfeng and Murtadha, Ahmed and Wen, Bo and Liu, Yunfeng},
  journal       = {Neurocomputing},
  volume        = {568},
  pages         = {127063},
  year          = {2024}
}

@inproceedings{vaswani2017attention,
  title         = {Attention Is All You Need},
  author        = {Vaswani, Ashish and Shazeer, Noam and Parmar, Niki and Uszkoreit, Jakob and Jones, Llion and Gomez, Aidan N. and Kaiser, Lukasz and Polosukhin, Illia},
  booktitle     = {Advances in Neural Information Processing Systems},
  volume        = {30},
  year          = {2017}
}

@inproceedings{xiao2024streamingllm,
  title         = {Efficient Streaming Language Models with Attention Sinks},
  author        = {Xiao, Guangxuan and Tian, Yuandong and Chen, Beidi and Han, Song and Lewis, Mike},
  booktitle     = {International Conference on Learning Representations},
  year          = {2024}
}

@inproceedings{zaheer2020bigbird,
  title         = {Big Bird: Transformers for Longer Sequences},
  author        = {Zaheer, Manzil and Guruganesh, Guru and Dubey, Avinava and Ainslie, Joshua and Alberti, Chris and Ontanon, Santiago and Pham, Philip and Ravula, Anirudh and Wang, Qifan and Yang, Li and Ahmed, Amr},
  booktitle     = {Advances in Neural Information Processing Systems},
  volume        = {33},
  pages         = {17283--17297},
  year          = {2020}
}

@inproceedings{zhang2024found,
  title         = {Found in the Middle: How Language Models Use Long Contexts Better via Plug-and-Play Positional Encoding},
  author        = {Zhang, Zhenyu and Chen, Runjin and Liu, Shiwei and Yao, Zhewei and Ruwase, Olatunji and Chen, Beidi and Wu, Xiaoxia and Wang, Zhangyang},
  booktitle     = {Advances in Neural Information Processing Systems (NeurIPS)},
  year          = {2024},
  eprint        = {2403.04797},
  archiveprefix = {arXiv},
  primaryclass  = {cs.CL}
}

@article{grattafiori2024llama3,
  title   = {The Llama 3 Herd of Models},
  author  = {Grattafiori, Aaron and Dubey, Abhimanyu and Jauhri,
             Abhinav and others},
  journal = {arXiv preprint arXiv:2407.21783},
  year    = {2024}
}

@misc{izacard2021contriever,
  title  = {Unsupervised Dense Information Retrieval with Contrastive Learning},
  author = {Izacard, Gautier and Caron, Mathilde and Hosseini, Lucas and Riedel, Sebastian and Bojanowski, Piotr and Joulin, Armand and Grave, Edouard},
  year   = {2021},
  note   = {arXiv:2112.09118},
  doi    = {10.48550/arXiv.2112.09118}
}

@inproceedings{yang2018hotpotqa,
  title     = {{HotpotQA}: A Dataset for Diverse, Explainable Multi-hop
               Question Answering},
  author    = {Yang, Zhilin and Qi, Peng and Zhang, Saizheng and
               Bengio, Yoshua and Cohen, William W. and
               Salakhutdinov, Ruslan and Manning, Christopher D.},
  booktitle = {Proceedings of the 2018 Conference on Empirical Methods
               in Natural Language Processing},
  pages     = {2369--2380},
  year      = {2018}
}
